\documentclass[final,12pt]{colt2024} 

\title[Optimal Sample Complexity for RL in Continuous-Space MDPs]{Projection by Convolution: Optimal Sample Complexity for Reinforcement Learning in Continuous-Space MDPs}
\usepackage{times}
\usepackage{amsmath}
\usepackage{amssymb}
\usepackage{mathtools}
\usepackage{algorithmic}
\usepackage{tikz}
\usepackage{pgfplots}
\pgfplotsset{compat=newest}
\usepgfplotslibrary{fillbetween}
\usepackage{xcolor}
\usepackage{thm-restate}

\definecolor{C1}{RGB}{255, 127, 14}
\definecolor{C2}{RGB}{0, 153, 136} 
\definecolor{C3}{RGB}{204, 51, 17} 
\definecolor{C4}{RGB}{170, 51, 119} 

\coltauthor{%
 \Name{Davide Maran} \Email{davide.maran@polimi.it}\\
 \AND
 \Name{Alberto Maria Metelli} \Email{albertomaria.metelli@polimi.it}\\
 \AND
 \Name{Matteo Papini} \Email{matteo.papini@polimi.it}\\
 \AND
 \Name{Marcello Restelli} \Email{marcello.restelli@polimi.it}\\
 \addr Politecnico di Milano, Piazza Leonardo da Vinci, 32-36 - Città Studi, Milano (MI)%
}

\newtheorem{defin}{Definition}

\newtheorem{thm}{Theorem}

\newtheorem{prop}[thm]{Proposition}

\newtheorem{ass}{Assumption}
\newtheorem{rem}{Remark}

\newcommand{\R}{\mathbb{R}}
\newcommand{\N}{\mathbb{N}}
\newcommand{\E}{\mathop{\mathbb{E}}}
\newcommand{\Prob}{\mathbb{P}}

\newcommand{\soc}{\text{soc}}

\newcommand{\wass}{\mathcal W}
\newcommand{\bigo}{\mathcal O}
\newcommand{\bigot}{\widetilde{\mathcal O}}
\newcommand{\bm}{\boldsymbol}
\DeclareMathOperator*{\argmax}{arg\,max}

\makeatletter
\newcommand*{\bdiv}{%
  \nonscript\mskip-\medmuskip\mkern5mu%
  \mathbin{\operator@font div}\penalty900\mkern5mu%
  \nonscript\mskip-\medmuskip
}
\makeatother

\DeclareMathOperator{\Tr}{Tr}

\newcommand{\Ss}{\mathcal{S}}
\newcommand{\As}{\mathcal{A}}
\newcommand{\Zs}{\mathcal{Z}}
\newcommand{\Xs}{\mathcal{X}}

\makeatletter
\newcommand*{\rom}[1]{\expandafter\@slowromancap\romannumeral #1@}
\makeatother

\usepackage[textsize=tiny,disable]{todonotes}

\begin{document}

\maketitle








\begin{abstract}%
We consider the problem of learning an $\varepsilon$-optimal policy in a general class of continuous-space Markov decision processes (MDPs) having smooth Bellman operators.
Given access to a generative model, we achieve rate-optimal sample complexity by performing a simple, \emph{perturbed} version of least-squares value iteration with orthogonal trigonometric polynomials as features. Key to our solution is a novel projection technique based on ideas from harmonic analysis. 
Our~$\widetilde{\mathcal{O}}(\epsilon^{-2-d/(\nu+1)})$ sample complexity, where $d$ is the dimension of the state-action space and $\nu$ the order of smoothness, recovers the state-of-the-art result of discretization approaches for the special case of Lipschitz MDPs $(\nu=0)$. At the same time, for $\nu\to\infty$, it recovers and greatly generalizes the $\mathcal{O}(\epsilon^{-2})$ rate of low-rank MDPs, which are more amenable to regression approaches. In this sense, our result bridges the gap between two popular but conflicting perspectives on continuous-space MDPs.  
\end{abstract}


\section{Introduction}

\emph{Reinforcement learning} \citep[RL,][]{sutton2018reinforcement} is a paradigm of artificial intelligence where an agent interacts with an environment, receiving a reward and observing transitions in the state of the environment. The agent's objective is to learn a policy, which is a mapping from the states to the probability distributions on the set of their action, that maximizes the expected return, i.e., the expected long-term cumulative reward. As the environment is stochastic, a way of assessing the performance of an algorithm, having fixed an accuracy threshold $\varepsilon$ and error probability $\delta$, is to find the minimum number of interactions $n$ with the environment needed to learn a policy that is $\varepsilon-$optimal with probability at least $1-\delta$. This kind of result falls under the name of \textit{sample complexity bounds}. For the case of \emph{tabular Markov decision processes}~\cite[MDPs,][]{puterman2014markov}, an optimal result was first proved by~\citet{gheshlaghi2013minimax}, showing a bound on the sample complexity with access to a generative model
of order $\mathcal {\widetilde O}(H^2|\Ss||\As|\varepsilon^{-2})$ with high probability, where $\Ss$ is a finite state space, $\As$ is a finite action space and. $H$ the time horizon of every episode. 
This regret is minimax-optimal, in the sense that no algorithm can achieve smaller regret for every arbitrary tabular MDP, as proved by \citet{gheshlaghi2013minimax}. 
Assuming that the MDP is tabular, with a finite number of states and actions, is extremely restrictive. Indeed, most appealing applications of RL like trading \citep{hambly2023recent}, robotics \citep{kober2013reinforcement}, and autonomous driving \citep{kiran2021deep} do not fall in this setting. For this reason, the search for algorithms with sample complexity bounds for RL in \emph{continuous} spaces is currently one of the most important challenges of the whole field. Obviously, it is not possible to learn an $\varepsilon-$optimal policy for any continuous space MDP\footnote{Think, for instance, of the case where the reward function is discontinuous.} and assumptions should be made on the structure of the process.

Previous works have introduced a variety of different settings when the continuous-space RL problem is learnable. The \emph{linear quadratic regulator} \citep[LQR,][]{bemporad2002explicit}, a model coming from control theory, assumes that the state of the system evolves according to a linear dynamics and that the reward is quadratic. For this problem, when the system matrix is unknown, \citet{abbasi2011regret} obtained a sample complexity bound of order $\bigot(\varepsilon^{-2})$ for a computationally inefficient algorithm. This latter limitation was then removed by \citet{dean2018regret, cohen2019learning}. 
\emph{Linear MDPs}~\citep{yang2019sample,jin2020provably} are another very popular setting where a different form of linearity is assumed. Here, the transition function of the MDP can be factorized as a scalar product between a feature map $\bm \varphi: \Ss \times \As \to \R^d$ and an unknown vector of signed measures over $\Ss$. The reward function is typically assumed to be linear in the same features. For this setting, assuming that the feature map is known, there are algorithms \citep{jin2020provably} achieving sample complexity of order $\bigot(\varepsilon^{-2})$. When the feature map is unknown \citep{agarwal2020flambe}, the problem gets much more challenging. State-of-the-art results in this field \citep{uehara2021representation} are able to achieve sample complexity $\bigot(|\As|^5\log(|\mathcal F|)\varepsilon^{-2})$, where $\mathcal F$ is a known function class containing the true feature map, only if the reward function is known. Despite the optimal dependence on $\varepsilon$, this bound comes with the additional intake that the action space must be finite and ``small'', so that this result does not apply to continuous-space RL.

All these settings, by assuming a parametric model, are relatively limited and clearly not able to provide a perspective over general MDPs with continuous spaces. A more general model, building on the often realistic assumption that close state-action pairs are associated with similar reward and transition functions, is that of \emph{Lipschitz MDPs}~\citep{rachelson2010locality}. Lipschitz MDPs have been applied to several different settings, such as policy gradient methods \citep{pirotta2015policy}, model-based RL~\citep{asadi2018lipschitz}, RL with delayed feedback~\citep{liotet2022delayed}, and auxiliary tasks such as imitation learning~\citep{maran2023tight}. The price for generality is paid with a very inconvenient sample complexity guarantee. State-of-the-art works \citep[e.g.,][]{song2019efficient,sinclair2019adaptive,le2021metrics} in learning theory which focused on this setting where only able to achieve sample complexity of order $\bigot(\varepsilon^{-d-2})$, where $d$ is the dimension of the state-action space, that is assumed to be $\Ss\times \As=[-1,1]^d$. Another non-parametric family of continuous spaces MDPs is that of \emph{Kernelized MDPs}~\citep{yang2020provably}, where both the reward function and the transition function belong to a \emph{reproducing kernel Hilbert space} (RKHS) induced by a known kernel. In the typical application to continuous-state MDPs, the kernel is assumed to come from the Matérn covariance function with parameter $m>0$. Optimal sample complexity bounds for this problem were very recently proved by \cite{vakili2023kernelized}, which showed a sample complexity of order $\bigot(\varepsilon^{-\frac{d+2m}{m}})$. In this result, the presence of $d$ at the exponent is mitigated by $m$. Indeed, for $m \gg d$, the bounds approach the desired value of $\bigot(\varepsilon^{-2})$ that is achievable for LQRs and Linear MDPs. Still, the Kernelized MDP model is significantly more restrictive than that of Lipschitz MDPs. For example, any deterministic MDP is not included in the kernelized family for the Mat\'ern kernel.

\paragraph{Our Contribution.} As we saw, the results in the literature of continuous-space RL are affected by a huge gap. Good sample complexity bounds of order $\bigot(\varepsilon^{-2})$ are known for very \emph{specific} classes of problems, and very unsatisfactory bounds of order $\bigot(\varepsilon^{-d-2})$ are known for the very \emph{general} Lipschitz MDP case. In this paper, we will focus on a very large class of MDPs known as weakly $\nu-$smooth MDPs, introduced in \citep{maran2023noregret}, depending on a smoothness parameter $\nu\in \N$. For $\nu=0$, this class is a generalization of the Lipschitz MDP, while for any $\nu$ (including $\nu\to +\infty$) this class generalizes the most common parametric problems, including LQRs and Linear MDPs. For any $\nu$, this class also generalizes the assumption of kernelized MDP when $m>\nu+1$. We then define an algorithm based on a novel regression technique, which draws on both kernel theory and Fourier series, able to achieve a sample complexity bound of order $\bigot(\varepsilon^{-\frac{d+2(\nu+1)}{\nu+1}})$, which is tight in all cases, matching one lower bound proved in the much simpler setting of \textit{stochastic optimization}. In this way, the algorithm is able to achieve the best order in $\varepsilon$ both for processes with little smoothness ($\approx \varepsilon^{-d-2}$) and with higher smoothness ($\approx \varepsilon^{-2}$).

\begin{rem}
    Note that, in the above literature review, some works that we have mentioned do not explicitly show sample complexity bounds, but rather regret bounds. Without giving other details, it is sufficient to say that any bound of the regret of the form $R_T\le \bigot(CT^{\beta})$ leads to a sample complexity guarantee of $n\le \bigot(C/\varepsilon)^{\frac{1}{1-\beta}}$, both holding in high probability (one just runs a regret minimization algorithm and outputs a policy drawn uniformly from the sequence of policies). Using this relation, we have rephrased all the results in the literature in terms of sample complexity.
\end{rem}

\section{Preliminaries} 

We consider a finite-horizon Markov decision process ~\citep[MDP,][]{puterman2014markov} $M=(\Ss, \As, p, R, H)$, where $\Ss = [-1,1]^{d_S}$ is the state space, and $\As = [-1,1]^{d_A}$ is the action space (this choice is without loss of generality, as any compact set could be used instead); $p=\{p_h\}_{h=1}^H$ is the sequence of transition functions, each mapping a pair $(s,a) \in \mathcal{S \times A}$ to a probability distribution $p_h(\cdot |s,a)$ over $\Ss$; $R=\{R_h\}_{h=1}^H$ is the sequence of reward functions, each mapping a pair $(s,a)$ to a real number $R_h(s,a)\in [0,1]$, and $H$ is the time horizon.
A policy $\pi = \{\pi_{h}\}_{h=1}^H$ is a sequence of mappings from $\Ss$ to the probability distributions over $\As$. For each stage  $h \in [H]$, the action is chosen according to $a_h\sim \pi_{h}(\cdot |s_h)$,  the agent gains a deterministic\footnote{This is just assumed for simplicity. %
Results are still valid if we only have access to noisy reward samples, as long as the noise is subgaussian since the stochasticity in the transition function overcomes the one in the reward.} reward $R_h(s_h,a_h)$, and the environment transitions to the next state $s_{h+1}\sim p(\cdot |s_h,a_h).$ In this setting, it is useful to define the following quantities.

\paragraph{Value functions and Bellman operator.}
The state-action value function (or \emph{Q-function}) quantifies the expected sum of the rewards obtained under a policy $\pi$, starting from a state-stage pair $(s,h)\in\Ss\times [H]$ and fixing the first action to some $a\in\As$:
\begin{align}\label{eq:action_state_val}
    Q_h^{\pi}(s,a) \coloneqq \mathbb{E}_{\pi} \left[ \sum_{\ell=h}^{H} R_\ell(s_\ell,a_\ell)\bigg | s_0=s,a_0=a  \right],
\end{align}
where $\E_{\pi}$ denotes expectation w.r.t. to the stochastic process $a_h \sim \pi_h(\cdot|s_h)$ and $s_{h+1} \sim p_h(\cdot|s_h,a_h)$ for all $h \in [H]$.
The state value function (or \emph{V-function}) is defined as $V_h^\pi(s)\coloneqq\E_{a\sim\pi_h(\cdot\vert s)}[Q_h^\pi(s,a)]$, for all $s\in\Ss$.
It can be proved that, under mild smoothness assumptions~\citep{bertsekas1996stochastic}, there is a family of \textit{optimal} policies such that their $Q^\pi_h$ functions at any stage dominates the one of every other policy for every state-action couple. Their corresponding state-action value function must satisfy the following condition:

$$\forall h\in [H+1]\qquad Q_h^*:\ \Ss \times \As \to \R,\qquad 
Q_h^*=
\begin{cases}
0&h=H+1\\
\mathcal T_hQ_{h+1}^*&1\le h\le H
\end{cases},$$

where $\mathcal T$ is the \textit{Bellman optimality operator}, defined in this way:
$$\mathcal T_h f(s,a):=R_h(s,a)+\E_{s'\sim p_h(\cdot|s,a)}\Big[\sup_{a'\in \As}f(s',a')\Big].$$

As the operator $\mathcal T$ chooses, at each time step, the action maximizing long-term expected reward, it is easy to prove that $Q_h^*(s,a)=\sup_{\pi} Q_h^\pi(s,a)$ for every $s,a,h$.

\paragraph{Agent's goal.}
In this paper, we assume to have access to a \emph{generative model} that, given a state-action-stage tuple $(s,a,h)$, returns a sample for the next state and one for the reward, $s'\sim p_h(\cdot|s,a)$, $\ r_h=R_h(s,a)$.
The agent's goal is to output an estimate of the optimal policy with the least possible number of queries to the generator. Precisely, one seeks to find a probably approximately correct (PAC) policy $\hat \pi$, in the sense that the policy must satisfy:
$$\|V_1(\cdot)-V^{\hat \pi}_1(\cdot)\|_{L^\infty}:=\sup_{s\in \Ss}|V_1(s)-V^{\hat \pi}_1(s)|\le \varepsilon,$$
with probability at least $1-\delta$. This shall be done with the least possible number of interactions with the environment (i.e., calls to the generative model), denoted with $n$.

\paragraph{Smoothness of real functions.}
Achieving a PAC guarantee for arbitrary MDPs with continuous spaces is clearly impossible. Indeed, even in the case $H=1$ and $R_1(s,a)=R_1(a)$ for all $s\in\Ss$ (a continuous-armed bandit), if $R_1(a)=\bold 1\{a=a_0\}$ (a discontinuous function), no algorithm can identify $a_0$, hence the optimal policy. For this reason, assumptions must be enforced to guarantee that the reward and transition functions are endowed with some regularity property. 
Let $\Omega \subset [-1,1]^d$ and $f : \Omega \to \R$. We say that $f\in \mathcal C^{\nu,1}(\Omega)$ if it is $\nu-$times continuously differentiable,
and

$$\|f\|_{\mathcal C^{\nu,1}}:=\max_{| \alpha|\le \nu+1}\|D^{\alpha} f\|_{L_\infty}<+\infty,$$

where the multi-index derivative is defined as follows
$$D^{\alpha}f:=\frac{\partial^{ \alpha_1+...+  \alpha_d}}{\partial x_1^{ \alpha_1}\dots \partial x_d^{ \alpha_d}}.$$

The most straightforward case of this definition is given by $\mathcal C^{0,1}(\Omega)$, corresponding to the space of Lipschitz continuous functions, whose increment is bounded by a constant.

Before digging into the application of smooth functions to MDPs, since we are going to draw from the theory of the Fourier series, we also have to introduce the periodic version of the spaces defined before. For $\Omega=[-1,1]^d$, we define the space $\mathcal C_p^{\nu,1}(\Omega)$ to be the set of functions that are $\mathcal C^{\nu,1}(\Omega)$ with periodic conditions at the boundaries (see for example \citet{wu2014applying}). For this space, we take same the norm $\|\cdot\|_{\mathcal C^{\nu,1}(\Omega)}$ defined above.
This assumption is necessary to use methods based on the Fourier series, as is often done in physics and engineering. Dealing with this kind of function, we can define the convolution even if the domain $\Omega$ is bounded. In the rest of this paper, we will call, for periodic functions $f,g$ on $\Omega$,

$$f*g(x):=\int_{\Omega}f(y)g(x-y)\ dy,$$

where $g$ is extended by periodicity (namely, if any component $i$ of $x-y$ is more than $1$, it becomes $x_i-y_i-2$, and the opposite for $x_i-y_i<-1$) when $x-y$ falls outside of $\Omega$. This operation is usually called \emph{circular convolution}.

\paragraph{Smoothness in MDPs.}

In previous works on MDPs with continuous spaces, smoothness was assumed on the transition function $p_h(s'|s,a)$ as a function of these three variables. Here, following the same idea of \citet{maran2023noregret}, we introduce a much more general assumption, which directly concerns the Bellman optimality operator.
\begin{ass}\label{ass:smooth}\textit{($\nu-$Smooth MDP)}. An MDPs is \emph{smooth} of order $\nu$ if, for every $h\in [H]$, the Bellman optimality operator $\mathcal{T}_h$ is bounded on $\mathcal C_p^{\nu,1}(\Ss \times \As) \to \mathcal C_p^{\nu,1}(\Ss \times \As)$ with a constant $C_{\mathcal T}$.
\end{ass}

Recall that, for an operator $\mathcal B:\ (\mathbb V,\|\cdot\|_{V})\to (\mathbb W,\|\cdot\|_{W})$ between two normed vector spaces $\mathbb V,\mathbb W$, boundedness means that there is a constant $C_{\mathcal B} < +\infty$ such that $\|\mathcal B v\|_{W} \le C_{\mathcal B} (\|v\|_{V}+1)$. Therefore, the operator cannot transform a function that has a small norm into one that has one that is arbitrarily large. In our case, we are assuming that $\mathcal{T}_hf\in \mathcal C_p^{\nu,1}(\Ss \times \As)$ with $\|\mathcal{T}_hf\|_{C^{\nu,1}}\le C_{\mathcal{T}}(\|f\|_{C^{\nu,1}}+1)$, for every function $f\in\mathcal C_p^{\nu,1}(\Ss \times \As)$ and every $h\in[H]$. 
This assumption coincides with the one of Weakly Smooth MDP by \citet{maran2023noregret}, apart from the fact that we are considering periodic functions (we have added the $p$ subscript in $\mathcal{C}_p^{\nu,1}$ as a reminder).
No assumption is strictly stronger than the other, and as the periodicity condition only concerns the boundaries of $\Ss \times \As$, the two models are essentially equivalent when the interesting part of the MDP lies in the interior of $\Ss \times \As$.

\section{Trigonometric representation}
The algorithm we present in this paper is based on the idea of approximating the state-action value function at any time-step $h$ with a trigonometric polynomial. We are going to use the following notation.
\begin{defin}
    Define the function $\soc:\mathbb Z\times [-1,1]\to[-1,1]$ as:
    $$\soc(n ,z):=
    \begin{cases}
        1& n=0\\
        \sin(n\pi z)& n>0\\
        \cos(n\pi z)& n<0.
    \end{cases}
    $$
    Overloading the notation, for $\bm n\in\mathbb Z^d$ and $ z\in[-1,1]^d$, define the monomial:
    \begin{equation*}
        \soc(\bm n, z)\coloneqq\soc(n_1, z_1)\soc(n_2, z_2)\dots \soc(n_d, z_d).
    \end{equation*}
    Define, for every $N>0$ and $d>0$, both integers, a degree $N$, $d-$variate trigonometric polynomial as a function $[-1,1]^d\to \R$ of the form
    $$f(z)=\sum_{\substack{\bm n \in \mathbb Z^d \\\|\bm n\|_1\le N}}\theta_{\bm n}\soc(\bm n, z),$$
    for some real coefficients $\theta_{\bm n}$. We call $\mathbb T_N^d$ the vector space of these functions ($\mathbb T_N$ when $d$ is clear from the context).
\end{defin}

For $d=1$, the space $\mathbb T_N$ corresponds to the vector space generated by the standard basis $1$, $\sin(\pi z)$, ... $\sin(N\pi z),\cos(N\pi z)$. For general $d$, the vector-space dimension of this space corresponds to the cardinality of the set $\{\bm n \in \mathbb Z^d:\ \|\bm n\|_1\le N\}$, which can be easily proved to be
$$\widetilde N:=|\{\bm n \in \mathbb Z^d:\ \|\bm n\|_1\le N\}|=\binom{2N+d}{d}\le (2N+1)^d.$$
From this basis, we can build a feature map by just stacking the terms $\bm n \in \mathbb Z^d:\ \|\bm n\|_1\le N$ in a vector of functions $e_1(z),\dots e_{\widetilde N}(z)$. The corresponding feature map is then

$$\bm \varphi_N: [-1,1]^d\to \R^{\widetilde N},\qquad \bm \varphi_N(x) = [e_1(z),\dots e_{\widetilde N}(z)]^\top.$$

Indeed, for every $N\in \N$, we can express any function $g\in \mathbb T_N$ as $g(z)=\bm \varphi_N(z)^\top \bm \theta_N(g),$
where the feature map depends on $z$ only and the coefficient vector $\bm \theta_N(g)$ on $g$ only. The good part about this is that we can do the same for functions that do not belong to $\mathbb T_N$, at the cost of having a residual term. 
In this way, any function $f$ may be decomposed as a sum of a term that is linear in the feature map and a residual term $\xi_N[f](z)$:
\begin{equation}f(z)=\bm \varphi_N(z)^\top \bm \theta_N(f)+\xi_N[f](z).\label{eq:linear}\end{equation}

In fact, $\bm \varphi_N(z)$ is a vector of trigonometric functions of $z$, while $\bm \theta_N(f)$ is the vector stacking all the coefficients, both having length $\widetilde N$, both according to the same arbitrary fixed order.
Differently from the case of $f\in \mathbb T_N$, the coefficient $\bm \theta_N$ is \textit{not} uniquely determined. If $f\notin \mathbb T_N$, this coefficient should be chosen so that some norm of the residual term $\xi_N[f](z)$ gets minimized. If we minimize the $L^2$-norm of this quantity, Fourier series \citep{katznelson2004introduction} are the way to go, but different projections on the space $\mathbb T_N$ result in different values for the parameter $\bm \theta_N$.
Crucially, one of the main features that all these projections share with the Fourier series is that the magnitude of the residual term $\xi_N(f)$ decreases with $N$ at a rate that depends on how smooth the function $f$ is, as shown by the following classical result. 
\begin{thm}\label{thm:liter}(Theorem 4.1 part (ii) from \cite{schultz1969multivariate})
    There exists an absolute constant $K>0$ such that, for any $f\in \mathcal \mathcal C_p^{\nu,1}([-1,1]^d)$ we have:
    $$\inf_{t_N\in \mathbb T_N}\|t_N(\cdot)-f(\cdot)\|_{L^\infty}\le \frac{K}{N^{\nu+1}}\|f\|_{\mathcal C^{\nu,1}},$$
    where $\mathbb T_N$ denotes the space of trigonometric polynomials of degree not higher than $N$.
\end{thm}

\subsection{Our Approach: Foundations}

The idea of the algorithm that we will use to tackle this problem is based on the connection between the decay of the Fourier coefficients and the smoothness of the Bellman operator. Indeed, assuming that the Bellman optimality operator is bounded in the space $\mathcal C_p^{\nu,1}(\Ss \times \As)$, as in Assumption~\ref{ass:smooth}, entails the following desirable property.
Assume that the optimal state-action value function at the step $h+1$ has the form
$Q_{h+1}^*(s,a) = \bm\varphi_N(s,a)^\top \bm \theta_{h+1}.$
Then, by assumption on the Bellmann operator, we have:

$$Q_{h}^*(s,a) = \mathcal T_h Q_{h+1}^*(s,a) = \mathcal T_h \underbrace{\bm\varphi_N(s,a)^\top \bm \theta_{h+1}}_{\in \mathcal C_p^{\nu,1}(\Ss \times \As)} \in \mathcal C_p^{\nu,1}(\Ss \times \As),$$

as $\bm\varphi_N(s,a)^\top \bm \theta_{h+1}$ is a trigonometric polynomial, so it is $\mathcal C_p^{\nu,1}$ a fortiori. Then, also $Q_{h}^*$ can be written as 
$Q_{h}^*(s,a) = \bm \varphi_N(s,a)^\top \bm \theta_{h}+\xi_N(Q_{h}^*)(s,a),$
where $\bm \theta_{h}\in \mathbb \R^{\widetilde N}$ and $\xi_N(Q_{h}^*)(\cdot)$ is small in infinity norm, according to Theorem \ref{thm:liter}. If $\xi_N(Q_{h}^*)(\cdot)$ were exactly zero, this MDP would be said to be linearly Bellmann complete w.r.t. the feature map $\bm \varphi_N(\cdot)$. Instead, due to the presence of this error term, the feature map is said to suffer from a \textit{misspecification}. Despite similar scenarios have been previously studied in the literature \citep{zanette2020learning}, this misspecification significantly impacts the type of guarantees we are able to achieve. In fact, even in the simpler bandit setting (when $H=1$) it is well known \citep{lattimore2020learning} that learning a $\varepsilon-$optimal policy, even with an arbitrary number of samples, is only possible when $\varepsilon > \sqrt{\widetilde N}\|\xi_N(Q_{1}^*)\|_{L^\infty}.$
The fact that the misspecification gets multiplied by $\sqrt{\widetilde N}$ in the sample complexity turns out not to be improvable in the general case~\citep{lattimore2020learning}. Still, this has catastrophic effects on our setting: in order to decrease $\xi_N(Q_{h}^*)$ we have to increase $N$, but this leads to an increase in $\widetilde N$, a trade-off that may be unsolvable.
For example, we know that $\widetilde N=\bigo(N^d)$, while from Theorem~\ref{thm:liter} $\|\xi_N(Q_{h}^*)\|_{L^\infty}=\bigo(N^{\nu+1})$ when $Q_{h}^*\in \mathcal C_p^{\nu,1}(\Ss\times \As)$. So,

$$d>2\nu+2\implies \sqrt{\widetilde N}\|\xi_N(Q_{h}^*)\|_{L^\infty}=\bigo(N^{d/2-\nu-1}) \ge \bigo(1).$$

In this case, in the product $\sqrt{\widetilde N}\|\xi_N(Q_{h}^*)\|_\infty$ cannot become arbitrarily small, so that learning an $\varepsilon-$optimal policy using algorithms that deal with misspecification for general feature maps is impossible.

For this reason, we have to change our perspective. Instead of trying to learn $Q_{h}^*$, we are going to focus on a projected version $\widetilde Q_{h}$. 
In what follows, we will define $\mathcal P_N$ as a projection operator mapping every function $f$ to a trigonometric polynomial of degree $N$ so that $\widetilde Q_{h}$ solves, for every $h=1,\dots H$,
$$\widetilde Q_{h} =\mathcal P_N[\mathcal T_h \widetilde Q_{h+1}](s,a),$$

with $\widetilde Q_{H+1}=Q_{H+1}^*=0$. In this way, the misspecification vanishes completely, as starting from $\widetilde Q_{h+1}(s,a)$, that is a trigonometric polynomial, $\widetilde Q_{h}(s,a)$ is parametrized in the same way. This means that, with respect to the feature map $\bm\varphi_N$, the operator $\mathcal P_N[\mathcal T_h]$ makes no misspecification, and, hopefully, we can achieve sample complexity bounds by estimating $\widetilde Q_{h}(\cdot)$ with classical results for the LSVI algorithm~\citep{munos2005error}.
Two problems need now to be addressed: i) Due to the nature of the problem, we only have access to samples of $s_h,a_h,s_{h+1},r_{h}$ from the generative model. 
How can we approximate the projection operator $\mathcal P_N$ in this setting? ii) How do we know that $\widetilde Q_{h}$ is close to $Q_{h}^*$ for every $h$? In the following subsection, we are going to show that both open points can be solved in a satisfactory way due to the particular properties of our trigonometric feature map so that using a different map in place of $\bm \varphi_N$ would not achieve the same results.

\subsection{Projection by Convolution}\label{sec:proconv}

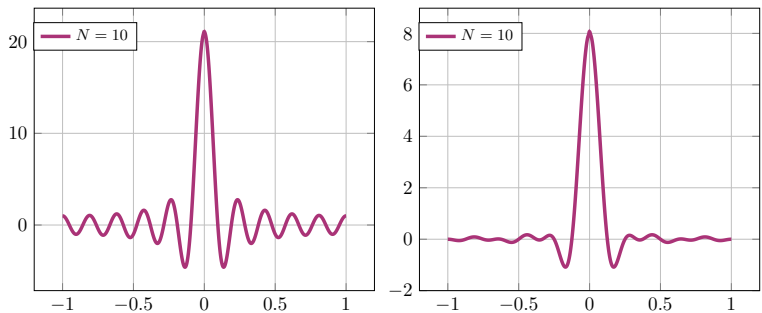
\begin{figure}[t]  
\centering 
\begin{tikzpicture} [scale=0.66]    
\begin{axis}[
    grid = both,
    legend style={at={(0.15,0.95)}, font=\footnotesize,anchor=north},
]
\addlegendentry{$N=10$}
\addplot[domain=-1:1,samples=500,smooth,C4, line width=2pt] {sin((10.5)*180*x)/sin(90*x)};
\end{axis} 
\end{tikzpicture} 
\begin{tikzpicture} [scale=0.66]    
\begin{axis}[
    grid = both,
    legend style={at={(0.15,0.95)}, font=\footnotesize,anchor=north},
]
\addlegendentry{$N=10$}
\addplot[domain=-1:1,samples=500,smooth,C4, line width=2pt] {(sin(180*(8*x))*sin(180*(3*x))/(12*sin(180*(x/2))^2)};
\end{axis} 
\end{tikzpicture}
\caption{A comparison between Dirichlet Kernel (left) and de
la Vall\'ee-Poussin one (right) in one dimension.} \label{fig:M}  
\end{figure}  

Let us suppose to have a black-box function $f: \Omega \to \R$, so that we are able to pick values $x_h \in \Omega$ and receive an unbiased sample $y_h = f(x_h)+\eta_h$, where $\eta_h$ is a zero-mean noise.
One simple trick shows that we can use randomization to obtain samples from functions that are different from $f$, still using the same generative model. Indeed, let $\xi_h$ be a noise with density function $g_{\xi}$. If, after choosing $x_h$, we perturb it with $\xi_h$, the resulting sample has a mean satisfying

$$\E[y_h] = \E[f(x_h+\xi_h)]=[f*g_{\xi}](x_h).$$

This process allows us to obtain unbiased samples from any function $\widetilde f$, as long as it can be expressed as a convolution between $f$ and another function. Thus, coming back to our problem, can we use this trick to obtain unbiased samples from 
$\mathcal P_N[f]$, for any choice of the function $f$?
The answer depends on the choice of the projection operator $\mathcal P_N$ that we want to use. If we are interested in the usual Fourier series projection, which minimizes the $L^2-$distance from the original function, the solution is provided by the Dirichlet kernel $D_N^0$ \citep{katznelson2004introduction}, a real function that, for $d=1$ takes the form $D_N^0(x)=\frac{\sin((N+1/2)\pi x)}{\sin(\pi x/2)}$. The main property of this kernel lays in the fact that, for \textit{any} function $f$, the convolution product $f*D_N^0$ corresponds exactly to the truncation to the order $N$ of the Fourier series associated to $f$. This is a possible solution for our main issue: $\|f-g\|_{L^2}$ is indeed minimized across all the trigonometric polynomials of degree $N$ when $g=f*D_N^0$. 

We can do even better. As we are interested in a trigonometric polynomial that approximates $f$ uniformly (not necessary in $L^2$), we may look for a different kernel that corresponds to a projection that is more suited for our case. This can be done by a slight modification of the Dirichlet kernel, which takes the name of de la Vallée Poussin kernel
\citep{de1918meilleure, de1919leccons,nemeth2016vallee}, 
that we will denote as $D_N$. This function is defined, for every positive even integer $N$, to be an average of Dirichlet Kernels for different orders $D_N(z):=\frac{1}{N/2+1}\sum_{n=N/2}^{N}D_n^0(z)$. The main feature that makes this kernel superior, for our purposes, to the Dirichlet kernel, is that \textit{its $L^1$ norm can be bounded regardless of $N$}. In fact, we have
\begin{equation}\|D_N\|_{L^1}\le \Lambda_d=\bigo(1),\qquad \|D_N^0\|_{L^1}=\bigo(\log(N)^d),\label{eq:boundD}\end{equation}
where the first part is rigorously proved in the appendix (Proposition \ref{prop:valprop}).
The same proposition allows us to show the following result, which expresses an important property of this kernel $D_N$.

\begin{restatable}{thm}{dieci}\label{thm:approxval}
    Let $f\in \mathcal C_p^{\nu,1}([-1,1]^d)$. Then, $D_N*f\in \mathbb T_{N}$, and we have
    $$\|f-D_N*f\|_{L^\infty} \le 2\Lambda_d\inf_{t_{N/2}\in \mathbb T_{N/2}}\|t_{N/2}-f\|_{L^\infty}\le \frac{2^{\nu+2}\Lambda_dK}{N^{\nu+1}}\|f\|_{\mathcal C^{\nu,1}},$$
    where $\mathbb T_N$ denotes the space of trigonometric polynomial of degree not higher than $N$ and $\Lambda_d$ comes from Equation~\eqref{eq:boundD}.
\end{restatable}

Clearly, the second inequality comes directly from Theorem \ref{thm:liter}. This result shows that the operator $D_N*f$ is able to achieve the best approximation of $f$ in $\mathbb T_N$, except for a constant factor $2^{\nu+2}\Lambda_d$.
This last result answers both questions of subsection \ref{sec:proconv}. Indeed, making the convolution $D_N*f$, which can be estimated by adding to $f$ a noise of density proportional to $D_N$, has also the role of a projection operator (as $D_N*f\in \mathbb T_{N}$) and good approximation properties ($\|f-D_N*f\|_{L^\infty}=\bigo(N^{-\nu-1})$).

\begin{algorithm2e}[t]
\everypar={\nl}
\RestyleAlgo{ruled}
\LinesNumbered
\caption{Quasi-optimal Design + Perturbed LSVI}\label{alg:DATA2}\label{alg:LSVI}
\DontPrintSemicolon
\KwData{State-action space $\Zs=[-1,1]^d$, reference value $n_{\text{tot}}$, approximation degree $N$}
\KwResult{Policy $\widehat \pi$.}
Let $\mathcal X_\varphi:=\{\bm \varphi_{N}(s,a):\ (s,a)\in \Zs\}$\label{algline:x}\;

Compute quasi-optimal design $\rho$ for the set $\mathcal X_\varphi$\;\label{algline:qo}

\For{$h = 1,\dots H$}{
    $\mathcal D_h \gets \emptyset$
    
    \For{$s,a\in \text{supp}(\rho)$}{
        Initialize $S_{s,a}$ with $\lceil n_{\text{tot}}\rho(s,a)\rceil$ copies of $s,a$\;\label{algline:collect}

        $\mathcal D_h \gets \mathcal D_h \cup S_{s,a}$ \;
    }
}
Find $D_N^+, D_N^-, \beta_+, \beta_-$ from Equation~\eqref{eq:deco}\label{algline:part2}\;

\For{$h = H,\dots 1$}{
  \For{$z_h=(s_h,a_h)\in \mathcal D_h$}{
    \textcolor{C2}{$\eta_+\sim D_N^+(\cdot)$}\;

    \textcolor{C3}{$\eta_-\sim D_N^-(\cdot)$}\;
    
    Query generative model for \textcolor{C2}{$s_{h+1}^+\sim p_h(\cdot|z_h+\eta_+),\ \ r_{h}^+\sim R_h(z_h+\eta_+)$}\label{algline:q1}\;

    Query generative model for \textcolor{C3}{$s_{h+1}^-\sim p_h(\cdot|z_h+\eta_-),\ \ r_{h}^-\sim R_h(z_h+\eta_-)$}\label{algline:q2}\;

    $Q^\textsc{target}_{h}(s_h,a_h) = \big(\beta_+r_h^+-\beta_-r_h^-\big) \:+\: \big(\beta_+\widehat V_{h+1}(s_{h+1}^+) -\beta_-\widehat V_{h+1}(s_{h+1}^-)\big)$
    \label{algline:unbiased}
    

  }  
  Solve least squares
  $$\widehat{\bm \theta}_h = \mathop{\arg\min}_{\bm \theta \in \R^{\widetilde N}}\sum_{(s_h,a_h)\in \mathcal D_h}\left (\bm \varphi_{N}(s_h,a_h)^\top \bm \theta
  -Q^\textsc{target}_{h}(s_h,a_h)\right)^2$$

  \textbf{Compute $\widehat Q_h(\cdot,\cdot)= {\bm \varphi}_{N}(\cdot,\cdot)^\top \widehat{\bm \theta}_h$}

  Compute next state-value function $\widehat V_h(\cdot)=\max_{a\in \As}\{ \bm \varphi_{N}(\cdot,a)^\top \widehat{\bm \theta}_h\}$  
}
Return $\widehat \pi_h(s):=\argmax_{a\in \As}\widehat Q_h(s,a)$
\end{algorithm2e}

\section{Algorithms}

In this section, we introduce the algorithm that we are going to use to achieve optimal sample complexity. Algorithm \ref{alg:LSVI} takes as input the state-action space dimension $d$, a reference value $n_{\text{tot}}$ for the number of interactions with the simulator that we are going to collect at each stage, and the approximation degree $N$ of the feature map.
The first part of the algorithm, until Line \ref{algline:part2} excluded, is devoted to choosing the state-action couples where to query the generative model. Here, the idea is to define a quasi-optimal design $\rho$ for least-squares regression over the set $\mathcal X_\varphi$ defined in Line \ref{algline:x}. Intuitively, a quasi-optimal design is a distribution over $\mathcal X_\varphi$ with few support points ($\bigo({\widetilde {N}}\log\log({\widetilde {N}}))$) such that its linear features "cover" all the ones in $\mathcal X_\varphi$. As this concept is relatively standard in RL theory, we leave further explanations to Appendix \ref{app:optimaldesign}. Once having chosen the design $\rho$, we collect in the set $S_{s,a}$ every point in the support of $\rho$ for a number of times given by $\lceil n_{\text{tot}}\rho(s,a)\rceil$. Note that, because of the ceiling, the total number of samples is not exactly $n_{\text{tot}}$, but slightly higher. From Line \ref{algline:part2} onward, we have just a perturbed version of the celebrated LSVI algorithm by~\citet{munos2005error}.
The main differences with respect to standard LSVI are highlighted in Algorithm \ref{alg:LSVI} in \textcolor{C2}{green} and \textcolor{C3}{red}. As we can see, instead of just querying the generative model for the elements of the buffers $\mathcal D_h$, we perturb them with a noise sampled from the positive and negative parts $\sim D_N^+, \sim D_N^-$ of the Valle\'e de la Poussin Kernel.
Ideally, this is done because we want unbiased samples from $D_N*\mathcal T_h \widetilde Q_{h+1}$, which can be done by adding a perturbation from $D_N(\cdot)$ to the state-action couple. However, we cannot take random samples from something that is not a density function, so we have to decompose $D_N$ in the following way:

\begin{equation}D_N(\cdot) = \beta_+D_N^+(\cdot)- \beta_-D_N^-(\cdot),\label{eq:deco}\end{equation}


with $D_N^+(z),D_N^-(z)$ being the positive and negative parts of the kernel, both normalized to have integral one, and $\beta^+,\beta^-$ are constants such that the equality is verified. 
Now, we can just take samples from $D_N^+, D_N^-$ and apply Equation~\eqref{eq:deco} to have an unbiased estimator (Line~\ref{algline:unbiased}) for $D_N*\mathcal T_h \widehat Q_{h+1}(s_h,a_h)$, which is what we need to keep our Q-function estimates within the vector space spanned by the features $\bm \varphi_{N}(\cdot,\cdot)$.
Note that, despite the modifications, as long as we have a good random number generator, the computational complexity of the algorithm does not increase significantly compared to LSVI: finding $\rho$ is computationally feasible, as shown by \citet[][Lemma 3.9]{todd2016minimum}, and the second part has the same complexity of LSVI. Additional reasonings of computational efficiency are left to the appendix \ref{app:compu}.

As anticipated, the reference value $n_{\text{tot}}$ does not correspond exactly to the number of samples that we take for each time-step $h$. In fact, we have that the actual number of samples corresponds to
$\sum_{s,a \in \text{supp}(\rho)}\lceil n_{\text{tot}}\rho(s,a)\rceil \le |\text{supp}(\rho)|+\sum_{s,a \in \text{supp}(\rho)} n_{\text{tot}}\rho(s,a)=|\text{supp}(\rho)|+n_{\text{tot}}.$ From this reasoning, we have the following proposition.

\begin{prop}\label{prop:totalsamp}
    The total number of samples required by Algorithm~\ref{alg:LSVI} in every phase %
    is~$2(n_{\text{tot}}+4{\widetilde N}\log\log({\widetilde N}))$.
\end{prop}
\begin{proof}
    The factor $n_{\text{tot}}+4{\widetilde N}\log\log({\widetilde N})$ directly comes from $n\le n_{\text{tot}}+|\text{supp}(\rho)|$ and the bound on the support given by definition of quasi-optimal design (Proposition \ref{thm:KW}). The factor $2$ comes from the double sampling in Lines \ref{algline:q1} and \ref{algline:q2}.
\end{proof}

\subsection{Theoretical Guarantees}

In this section, we introduce the main result of this paper, which ensures that our algorithm can be used on a Smooth MDP to achieve order-optimal sample complexity.

\begin{restatable}{thm}{mainthm}\label{thm:LSVIbound}
    Fix $\varepsilon>0,\delta>0$. Apply Algorithm \ref{alg:LSVI} on a $\nu-$Smooth MDP. Then, for a proper choice of $N=\bigo(\varepsilon^{-\frac{1}{\nu+1}})$, with probability at least $1-\delta$, we are able to learn a policy $\widehat \pi$ that is $\varepsilon-$optimal with a number of samples given by
    $$n= \bigot \left (K(H,d)\varepsilon^{-\frac{d+2(\nu+1)}{\nu+1}}\right ).$$
\end{restatable}

The full proof is very long, and we postpone it to Appendix \ref{app:finalproof}.
Unfortunately, the constant $K(H,d)$ from the previous theorem is exponential in both $H$ and $d$. In the next section, after proving that the exponent ${-\frac{d+2(\nu+1)}{\nu+1}}$ is tight, and no better result can be achieved, we also discuss this drawback.

\subsection{Tightness of the result and comparison with state of the art}\label{sec:tight}

In this section, we discuss the result of Theorem \ref{thm:LSVIbound}. We start from the question about the tightness of the order of the sample complexity, for which a lower bound from the field of stochastic optimization proves that our exponent is the best possible.

\begin{thm}\label{thm:wang}(\cite{wang2018optimization} Theorem 2 + Proposition 3 for $\beta=0$)
    Consider the optimization of an unknown function $f\in \mathcal C^{\nu,1}([-1,1]^d)$ via an algorithm that can only access $f$ trough noisy samples of the form $y_t=f(x_t)+\eta_t$, where $x_t$ are points chosen by the algorithm and $\{\eta_t\}_{t=1}^n$ is a sequence of independent $1-$subgaussian noises. Then, the expected number of samples required to find a $\varepsilon-$optimal point $\widehat x$ satisfies $n = \Omega\Big(\varepsilon^{-\frac{d+2(\nu+1)}{\nu+1}}\Big)$.     
\end{thm}

Note that the setting of the previous theorem is far easier than ours. First, the previous theorem holds in expectation, while our Theorem \ref{thm:LSVIbound} is valid in high probability. Second, and more importantly, the RL setting with a generative model reduces to stochastic optimization when $H=1$, so that the reward function $r_1(\cdot)$ corresponds to the function $f(\cdot)$ to be optimized in Theorem~\ref{thm:wang}. Therefore, Theorem \ref{thm:wang} implies \textit{a fortiori} a lower bound of order $n = \Omega\Big(\varepsilon^{-\frac{d+2(\nu+1)}{\nu+1}}\Big)$ in our setting, which matches our result for what concerns the dependence on $\varepsilon$. This result is completely new in the literature of RL with continuous spaces. Indeed, by converting regret bounds into sample complexity bounds, on the one hand, we have algorithms for Lipschitz MDPs \citep{song2019efficient,sinclair2019adaptive,le2021metrics}, which achieve only $\bigot(\varepsilon^{-d-2})$ whatever the smoothness of the process, on the other hand, \textsc{Legendre-Eleanor} \citep{maran2023noregret} which only works for $d<2\nu+2$ and gives $\bigot(\varepsilon^{-\frac{d+2(\nu+1)}{-d/2+\nu+1}})$, which is never tight.

Now, we move on to explore the dependence of the sample complexity on the task horizon $H$ and the dimensionality of the state-action space $d$. 
The exponential dependence found in our Theorem \ref{thm:LSVIbound} is not encouraging, but we can prove that this annoying phenomenon cannot be avoided in the whole family of Lipschitz MDPs (defined in the Appendix \ref{app:tight}), which are a subset of the Smooth MDPs for $\nu=0$ \citep{maran2023noregret}.
Indeed, for what concerns the time horizon $H$ and the dimension $d$, the following theorem, which is analogous to Theorem 12 in \citet{maran2023noregret}, shows that for any setting containing Lipschitz MDP, the exponential dependence on both $d$ and $H$ is necessary.

\begin{thm}\label{thm:vecio}
    Every algorithm for Lipschitz MDPs requires a sample complexity of order $\Omega(C^{dH}\varepsilon^{-d-2})$, for some constant $C>1$ depending on the process.
\end{thm}

Proof can be found in the Appendix \ref{app:tight}.
Note that all the results in the state of the art obtaining bounds that are sub-exponential in $H$ often hide this dependence in a problem-dependent constant. For example, Theorem 4.1 in \citet{sinclair2019adaptive} provides a regret bound depending on a constant $c$, that for Lipschitz MDPs reveals in Proposition 2.5 to be exponential in $H$. 
For what concerns the dependence on $H$, 
this result constitutes a significant problem, as it completely prevents Lipschitz MDPs with long horizons from being solved efficiently without further assumptions. 

\paragraph{Extension for misspecified $\nu$.}

As we have proved in Theorem \ref{thm:LSVIbound}, our algorithm is able to achieve optimal sample complexity for every $\nu-$Smooth MDP. Still, to achieve this result, an exact knowledge of the smoothness parameter $\nu$ is required, as it is not possible to compute the optimal number of features $N$ without it. In many practical scenarios, the knowledge of the exact $\nu$ is rather unrealistic, so we study what happens in case the algorithm is run for a parameter $\nu_{\text{miss}}$ that is different from the correct value $\nu$. Taking $\nu_{\text{miss}}>\nu$ leads to the impossibility of guaranteeing the learning of an $\varepsilon-$optimal policy. Since, as we prove in the appendix (theorem \ref{thm:approxQ}), the error $\|\widetilde Q_h- Q_h^*\|_\infty\le \Psi_h N^{-\nu-1}$, taking $N=\bigo(\varepsilon^{-\frac{1}{\nu_{\text{miss}}+1}})$ as prescribed by theorem \ref{thm:LSVIbound}, leads to

$$\forall \varepsilon<1\qquad \|\widetilde Q_h- Q_h^*\|_\infty = \bigo \left(\varepsilon^{\frac{\nu+1}{\nu_{\text{miss}}+1}}\right)>\bigo(\varepsilon).$$

Clearly, this kind of approximation does not allow us to learn a $\varepsilon-$optimal policy, even if $\widetilde Q_h$ were known exactly. In the opposite case, when $\nu_{\text{miss}}<\nu$, we can just exploit the fact that $\mathcal C^{\nu,1}(\Ss,\As)\subset \mathcal C^{\nu_{\text{miss}},1}(\Ss,\As)$ and apply the bound we would obtain if the MDP were only $\nu_{\text{miss}}-$Smooth. This writes as
$$n=\bigot\Big(\varepsilon^{-\frac{d+2(\nu_{\text{miss}}+1)}{\nu_{\text{miss}}+1}}\Big),$$
which is obviously worse than the one for $\nu$, but still finite.

\section{Conclusions and future works}

In this work, we have significantly enlarged the class of continuous-spaces MDPs for which sample-efficient learning with a generative model is possible. We focus on the Smooth MDP setting, defined first \cite{maran2023noregret}, which includes the majority of the classes of continuous spaces MDPs for which sample efficient RL is possible. 
To solve this very general problem, we design an algorithm that is able to achieve optimal sample complexity in $\varepsilon$ for every dimension $d$ and smoothness parameter $\nu$, thus, reaching a lower bound that cannot be improved even in the simpler bandit scenario. This algorithm builds on a novel regression technique that perturbs the sample with a noise distribution coming from the De La Vall\'ee Poussin kernel, which is a contribution of independent interest. The sample complexity bound is exponential in both $d$ and $H$, but we show a lower bound ensuring that, even in simpler settings, this drawback cannot be avoided.

\paragraph{Future work.} This work heavily relies on the assumption of interacting with a simulator during learning. A next advance could be to study if the same sample complexity can be achieved by just interacting online with the MDP. Moreover, a natural question in that setting is if the sample complexity bound can be turned into a regret bound. An intermediate step could be to study the applicability of this method to \emph{local access}---an intermediate setting between generative model and online RL, where the agent can reset the environment to previously visited states~\citep{yin2023sample}. Another potential application is to \emph{offline RL}, where the dataset is given in advance~\citep{levine2020offline}. Finally, it remains to be seen whether the dependence on the task horizon can be improved for interesting sub-families of smooth MDPs.

\acks{We thank the prof Maura Elisabetta Salvatori from the Department of Mathematics at Milano University for having introduced the first author of this paper to the de la Vallè Poussin kernel.}

\newpage
\bibliography{yourbibfile}

\clearpage
\appendix


\section{Harmonic Analysis}

\subsection{Formal definition of de
la Vall\'ee-Poussin kernel}

Formally, de
la Vall\'ee-Poussin kernels indicate a family of functions with two different parameters, $n,m$, which balance between some properties of the kernel~\citep[see][Equation 2.6]{nemeth2016vallee}. In this paper, we always fix $n=m=N$, so that the kernel can be written as: 

\begin{equation}\label{eq:val}
    D_{2N}:= \frac{1}{N+1}\sum_{n=N}^{2N}D_n^0,
\end{equation}

where $D_n^0$ denotes the order-$n$ standard Dirichlet kernel. This equation shows that the de
la Vall\'ee-Poussin kernel is actually just a superposition of Dirichlet kernels of different orders, and some of its properties follow from this. In particular, we need the following result, which is completely based on theorems discovered by \citet{mehta2015l1} and~\citet{nemeth2016vallee}.

\begin{prop}\label{prop:valprop}
    The following properties hold for the de la Vall\'ee-Poussin kernel (Equation \ref{eq:val})
    \begin{enumerate}
        \item For any periodic function $f:\Omega \to \R$, $D_N*f$ is a trigonometric polynomial of degree not higher than $N$.
        \item For any trigonometric polynomial $f$ of degree not higher than $N/2$, $D_{N}*f=f$.
        \item There is a constant $\Lambda_d$ depending only on the dimension $d$ such that $\|D_N\|_{L^1} \le \Lambda_d$.
    \end{enumerate}
\end{prop}
\begin{proof}
    We prove the thesis point by point.   \begin{enumerate}
        \item By linearity of the convolution, we have, from Equation \eqref{eq:val}:
        \begin{align*}
            D_{N}*f= \frac{1}{N/2+1}\sum_{n=N/2}^{N}D_n^0*f.
        \end{align*}
        By the properties of the Dirichlet kernel, $D_n^0*f$ is the Fourier series of $f$ truncated on $n$. As $n\le N$, all of these functions are trigonometric polynomials of degree not higher than $N$.
        \item Let $f$ be any trigonometric polynomial $f$ of degree not higher than $N/2$, then from Equation \eqref{eq:val}:
        \begin{align*}
            D_{2N}*f&= \frac{1}{N/2+1}\sum_{n=N/2}^{N}D_n^0*f\\
            & = \frac{1}{N/2+1}\sum_{n=N/2}^{N}f = f,
        \end{align*}
        where the second passage comes from the fact that the Fourier series of order $n\ge N$ of $f$ corresponds exactly to $f$ itself.
        \item We apply Theorem 2.4 from \cite{nemeth2016vallee}. Indeed, we have the following chain of inequalities:
        \begin{align*}
            \|D_N\|_{L^1} &= \sup_{g\in \mathcal C^0(\Omega), \|g\|_\infty\le 1}\int_\Omega D_N(x)g(x)\ dx\\
            &= \sup_{g\in \mathcal C^0(\Omega), \|g\|_\infty\le 1}\sup_{y\in \Omega}\int_\Omega D_N(x)g(y-x)\ dx\\
            &= \sup_{g\in \mathcal C^0(\Omega), \|g\|_\infty\le 1}\sup_{y\in \Omega}|D_N*g(y)|\\
            &= \sup_{g\in \mathcal C^0(\Omega), \|g\|_\infty\le 1}\|D_N*g\|_{\infty}.
        \end{align*}
        
        The latter corresponds exactly to the norm of the operator $\cdot \to D_N*\cdot$ over the space of continuous functions, which is bounded in Theorem 2.4 by \cite{nemeth2016vallee} as
        $$\sup_{g\in \mathcal C^0(\Omega), \|g\|_\infty\le 1}\|D_N*g\|_{\infty} \le c\left (\left(\log\frac{2N}{N+1}\right)^d+1\right)\le c(\log^d 2+1).$$
        Taking $\Lambda_d:=c(\log^d 2+1)$ ends the proof.
    \end{enumerate} 
\end{proof}

\dieci*
\begin{proof}
    The fact that $D_N*f$ is a trigonometric polynomial of degree not higher than $N$ is part one of Proposition \ref{prop:valprop}.
    Let $t_{N/2}\in \mathbb T_{N/2}$, we have:
    \begin{align*}
        \|f-D_N*f\|_{L^\infty} &= \|f-t_{N/2}\|_{L^\infty} + \|t_{N/2}-D_N*f\|_{L^\infty}\\
        & = \|f-t_{N/2}\|_{L^\infty} + \|D_N*t_{N/2}-D_N*f\|_{L^\infty}.        
    \end{align*}
    Here, the last passage comes from part two of Proposition \ref{prop:valprop}. Lastly, we have
    $$\|D_N*t_{N/2}-D_N*f\|_{L^\infty}\le \|D_N\|_{L^1}\|t_{N/2}-f\|_{L^\infty} \le \Lambda_d\|t_{N/2}-f\|_{L^\infty},$$
    by the third part of Proposition \ref{prop:valprop}. In the end, we have proved that, being $\Lambda_d\ge 1$,
    \begin{align*}
        \|f-D_N*f\|_{L^\infty}& \le \|f-t_{N/2}\|_{L^\infty} + \|D_N*t_{N/2}-D_N*f\|_{L^\infty}\\
        &\le 2\Lambda_d\|t_{N/2}-f\|_{L^\infty}.
    \end{align*}
    The proof is completed by applying Theorem \ref{thm:liter}.
\end{proof}

\section{Optimal Design for the Least Squares Problem}\label{app:optimaldesign}

The construction of the datasets $\mathcal D_h$ is one crucial part of Algorithm \ref{alg:LSVI}. This will be done in a way that is similar to \cite{lattimore2020learning}, by applying the concept of optimal design for the least squares estimator. To this aim, we recall the following theorem, which can be found in \citep[][Theorem 4.3]{lattimore2020learning} and is based on a previous result by \cite{kiefer1960equivalence}.

\begin{thm} (\cite{lattimore2020learning}, Theorem 4.4)\label{thm:KW}
    Suppose $\mathcal X \subset \R^{\widetilde N}$ is a compact set spanning $\R^{\widetilde N}$. Then, there is a probability distribution $\rho$ on $\mathcal X$ such that $|\text{supp}(\rho)|\le 4{\widetilde N}\log\log({\widetilde N})$ and, once defined
    $$\Sigma = \E_{\bm x\sim \rho}[\bm x \bm x^\top],$$
    we have, for all $\bm x \in \Omega$, $\|\bm x\|_{\Sigma^{-1}}^2\le 2{\widetilde N}$ (here the notation $\|\bm x\|_{\Sigma^{-1}}$ stands for $\sqrt{\bm x^\top \Sigma^{-1} \bm x}$).
\end{thm}

We call the distribution $\rho$ defined in Theorem \ref{thm:KW} \textit{quasi-optimal design} for the least square problem. The reason under the term "quasi" in the previous definition comes from a previous result. In fact, in a classical and celebrated paper,~\cite{kiefer1960equivalence} showed that, among all distribution over $\mathcal{X}$, 
we have always $\|\bm x\|_{\Sigma^{-1}}^2\ge {\widetilde N}$. Moreover, a distribution $\rho$ with at most $\widetilde N(\widetilde N+1)$ support points can be found such that the equality $\|\bm x\|_{\Sigma^{-1}}^2= {\widetilde N}$ is attained. This distribution takes the name of \textit{optimal design}. Instead, as we need a support much smaller than $\widetilde N(\widetilde N+1)$, we have relied on quasi-optimal design, which ensures $|\text{supp}(\rho)|\le 4{\widetilde N}\log\log({\widetilde N})$, still having $\|\bm x\|_{\Sigma^{-1}}^2\le 2{\widetilde N}$.

\section{Proof of Theorem \ref{thm:LSVIbound}}\label{app:finalproof}

We start with a preliminary result that will be used to bound the difference between $\widetilde Q_h$ and the original $Q_h$.

\begin{thm}\label{thm:approxQ}
    Let us assume to be in a $\nu-$Smooth MDP (Assumption \ref{ass:smooth}). Let $\widetilde Q_h$ be the sequence of perturbed state action value functions, satisfying
    $\widetilde Q_h(s,a)=[D_N*\mathcal T_h \widetilde Q_{h+1}](s,a).$
    Then,
    $$1)\ \|\widetilde Q_h\|_{\mathcal C^{\nu,1}}\le \frac{(\Lambda_dC_{\mathcal T})^{H-h+1}-\Lambda_dC_{\mathcal T}}{\Lambda_dC_{\mathcal T}-1},\qquad 2)\ \|\widetilde Q_h- Q_h^*\|_\infty\le \Psi_h N^{-\nu-1}$$
    where the constant $\Psi_h$ is defined as
    $$\Psi_h:= 2^{\nu+2}K(H-h+1)\frac{(\Lambda_dC_{\mathcal T})^{H-h+1}-\Lambda_dC_{\mathcal T}}{\Lambda_dC_{\mathcal T}-1}.$$
\end{thm}
\begin{proof}
    We perform the proof by backward induction on $h$, proving the two parts of the thesis at the same time.\\
    \textbf{Base case}: for $h=H+1$, both $\widetilde Q_h, Q_h^*$ are identically zero, so both the theses hold.\\    
    \textbf{Inductive case}: 
    We start from the first thesis.
    \begin{align}
        \|D_N*f\|_{\mathcal C^{\nu,1}} &= \max_{| \bm\alpha|\le \nu+1}\|D^\alpha (D_N*f)\|_{L^\infty}\nonumber\\
        & = \max_{| \bm\alpha|\le \nu+1}\|D_N*D^\alpha f\|_{L^\infty}\nonumber\\
        & \le \|D_N\|_{L^1}\max_{| \bm\alpha|\le \nu+1}\|D^\alpha f\|_{L^\infty} = \|D_N\|_{L^1}\|f\|_{\mathcal C^{\nu,1}}\label{eq:Dprop}.
    \end{align}

    Here, the first statement follows from the definition of norm in $\mathcal C^{\nu,1}$, the second by the fact that the derivative of a convolution can be applied to any of the two functions, the third from the inequality $\|f*g\|_{L^\infty}\le \|f\|_{L^\infty}\|g\|_{L^1}$. 

    Applying the previous inequality at passage $(\star)$ proves that
    
    \begin{align*}
        \|\widetilde Q_h\|_{\mathcal C^{\nu,1}} &= \|D_N*\mathcal T_h \widetilde Q_{h+1}\|_{\mathcal C^{\nu,1}}\\
        & \overset{(\star)}{\le}  \|D_N\|_1\|\mathcal T_h \widetilde Q_{h+1}\|_{\mathcal C^{\nu,1}}\\
        & \le \|D_N\|_1C_{\mathcal T} (\|\widetilde Q_{h+1}\|_{\mathcal C^{\nu,1}}+1)\\
        & \le \Lambda_dC_{\mathcal T} (\|\widetilde Q_{h+1}\|_{\mathcal C^{\nu,1}}+1)\\
        & \le \Lambda_dC_{\mathcal T} \left(\frac{(\Lambda_dC_{\mathcal T})^{H-h}-\Lambda_dC_{\mathcal T}}{\Lambda_dC_{\mathcal T}-1}+1\right)\\
        & =\Lambda_dC_{\mathcal T} \left(\frac{(\Lambda_dC_{\mathcal T})^{H-h}-1}{\Lambda_dC_{\mathcal T}-1}\right)\\
        & = \frac{(\Lambda_dC_{\mathcal T})^{H-h+1}-\Lambda_dC_{\mathcal T}}{\Lambda_dC_{\mathcal T}-1}.
    \end{align*}
    Where we have used induction in the fourth inequality.
    The proof of this point is completed by bounding $\|D_N\|_{L^1}$ with Proposition \ref{prop:valprop}.
    
    Let us turn to the second point, and note that
    \begin{align*}
        \widetilde Q_h- Q_h^* &= D_N*\mathcal T_h \widetilde Q_{h+1}-\mathcal T_h Q_{h+1}^*\\
        &=D_N*\mathcal T_h \widetilde Q_{h+1}-\mathcal T_h \widetilde Q_{h+1}+\mathcal T_h \widetilde Q_{h+1}-\mathcal T_h Q_{h+1}^*.
    \end{align*}
    Passing to the norms, $\|\mathcal T_h \widetilde Q_{h+1}-\mathcal T_h Q_{h+1}^*\|_\infty \le \|\widetilde Q_{h+1}- Q_{h+1}^*\|_\infty,$
    as the Bellman optimality operator is non-expansive in infinity norm. Coming to the other part we know, by the assumption of smoothness that the Bellman optimality operator is bounded over $\mathcal C^{\nu,1}(\Ss \times \As) \to \mathcal C^{\nu,1}(\Ss \times \As)$, so that $        \|\mathcal T_h \widetilde Q_{h+1}\|_{\mathcal C^{\nu,1}} \le C_{\mathcal T} (\|\widetilde Q_{h+1}\|_{\mathcal C^{\nu,1}}+1)$.
    Therefore, we have, by theorem \ref{thm:approxval}:
    \begin{align*}
        \|D_N*\mathcal T_h \widetilde Q_{h+1}-\mathcal T_h \widetilde Q_{h+1}\|_\infty &\le \frac{2^{\nu+2}\Lambda_dK}{N^{\nu+1}}\|\mathcal T_h \widetilde Q_{h+1}\|_{\mathcal C^{\nu,1}}\\
        &\le C_{\mathcal T} \left (\frac{2^{\nu+2}\Lambda_dK}{N^{\nu+1}}\|\mathcal T_h \widetilde Q_{h+1}\|_{\mathcal C^{\nu,1}}+1\right).
    \end{align*}
    Putting all pieces together ends the proof.
\end{proof}

We prove a second theorem which is the main ingredient of the central proof.

\begin{thm}\label{thm:inductor}
Fix $0<\varepsilon<1,\delta>0$. Apply Algorithm \ref{alg:LSVI} on a $\nu-$Smooth MDP. Then, for a proper choice of $N=\bigo(\varepsilon^{-\frac{1}{\nu+1}})$, with probability at least $1-\delta$, we are able to learn a function $\widehat Q_{h}$ such that $\| Q_{h}^*-\widehat Q_{h}\|_{L^\infty}<\varepsilon_h$ and $\| \widetilde Q_{h}-\widehat Q_{h}\|_{L^\infty}<\varepsilon_h$ for
$$\varepsilon_h:=\frac{\varepsilon}{(3\Lambda_d)^{h}},$$
With a number of samples per layer $h$ given by
    $$n= \bigot \left (K(H,d)\varepsilon^{-\frac{d+2(\nu+1)}{\nu+1}}\right ),$$
where $K(H,d)$ is a constant that is exponential in both quantities.
\end{thm}
\begin{proof}
    We prove the theorem by induction. Case $h=H+1$ is trivial, as both functions $\widetilde Q_{H+1},\widehat Q_{H+1}$ are identically zero.

    \textbf{Part \rom{1}: Bounding the Bellman error}. By design, at time step $h$, the algorithm performs a linear regression from $\bm \varphi_N(s,a)$ to $D_N*\mathcal T_h \widehat Q_{h+1}(s,a)$. Let us call $z_h=(s_h,a_h)$ an action pair in the dataset $\mathcal D_h$. The training set for the regression takes the form:
    $$\{\bm \varphi_N(z_h);\ Q^\textsc{target}_{h}[z_h]\}_{z_h\in \mathcal D_h}=\{\bm \varphi_N(z_h);\ \widetilde r_{h+1}[z_h] + \widetilde V_{h+1}[z_h]\}_{z_h\in \mathcal D_h},$$
    where we have split $Q^\textsc{target}_{h}[z_h]$ in two terms,
    $$\widetilde V_{h+1}[z_h]=\beta_+V_{h+1}(s_{h+1}^+) -\beta_-V_{h+1}(s_{h+1}^-)\qquad \widetilde r_{h+1}[z_h]=\beta_+r_h^+-\beta_-r_h^-,$$
    with $V_h(\cdot)=\max_{a\in \As} \bm \varphi_N(\cdot,a)^\top \widehat {\bm \theta}_{h+1}=\max_{a\in \As} \widehat Q_{h+1}(\cdot,a)$,
    and $s_{h+1}^+\sim p_h(\cdot|z_h+\eta_+),\  r_{h}^+= R_h(z_h+\eta_+)$ and $s_{h+1}^-\sim p_h(\cdot|z_h+\eta_-),\ r_{h}^-= R_h(z_h+\eta_-)$. By definition of the injected noise we have:
    \begin{align*}
        \E[\widetilde r_{h+1}[z_h]| z_h] &= \E[\beta_+r_h^+-\beta_-r_h^-| z_h]\\
        &=D_N*(R_{h+1})(z_h),
    \end{align*}
    while for the other part,
    \begin{align*}
        \E[\widetilde V_{h+1}[z_h] | z_h]&=\E[\beta_+V_{h+1}(s_{h+1}^+) -\beta_-V_{h+1}(s_{h+1}^-)| z_h]\\
        &=\beta_+\E[V_{h+1}(s_{h+1}^+)| z_h] -\beta_-\E[V_{h+1}(s_{h+1}^-)| z_h]\\
        &=\beta_+\E_{\substack{s_{h+1}^+\sim p_h(\cdot|z_h+\eta_+)\\ \eta_+\sim D_N^+}}[\max_{a\in \As} \widehat Q_{h+1}(s_{h+1}^+,a)| z_h]\\
        &\qquad -\beta_-\E_{\substack{s_{h+1}^-\sim p_h(\cdot|z_h+\eta_-)\\ \eta_-\sim D_N^-}}[\max_{a\in \As} \widehat Q_{h+1}(s_{h+1}^-,a)| z_h]\\
        &=\beta_+\int_{\Ss\times \As} D_N^+(y)\int_{\Ss}\max_{a\in \As} \widehat Q_{h+1}(x,a)p(x|z_h+y)dx\ dy\\
        &\qquad -\beta_-\int_{\Ss\times \As} D_N^-(y)\int_{\Ss}\max_{a\in \As} \widehat Q_{h+1}(x,a)p(x|z_h+y)dx\ dy\\
        &=\int_{\Ss\times \As} D_N(y)\int_{\Ss}\max_{a\in \As} \widehat Q_{h+1}(x,a)p(x|z_h+y)dx\ dy\\
        & = D_N*[\max_{a\in \As} \widehat Q_{h+1}](z_h).
    \end{align*}

    Where the first passage is by definition of $\widetilde V_h$, the third by definition of $V_h$, the fourth by rewriting the expectation as an integral and the last one by definition of convolution (kernel $D_N$ is symmetric, so adding or subtracting makes no difference). By definition of Bellmann operator, this proves that

    \begin{equation}
        \E[\widetilde r_{h+1}[z_h]+\widetilde V_{h+1}[z_h]| z_h] = D_N*\mathcal T\widehat Q_{h+1}(z_h)\label{eq:unbiased}.
    \end{equation}

    Moreover, note that we have
    
    \begin{align*}       
    \|\widetilde V_{h+1}(\cdot)\|_{L^\infty}&\le  \|\beta_+V_{h+1}(\cdot)\|_{L^\infty} +\|\beta_-V_{h+1}(\cdot)\|_{L^\infty}\\
    & \le (\beta_++\beta_-)\|V_{h+1}(\cdot)\|_{L^\infty}\\
    &\le \Lambda_d \|\widehat Q_{h+1}(\cdot)\|_{L^\infty}\\
    &\le \Lambda_d (\|\widehat Q_{h+1}(\cdot)-Q_{h+1}^*(\cdot)\|_{L^\infty}+\|Q_{h+1}^*(\cdot)\|_{L^\infty})\\
    &\le \Lambda_d (\varepsilon_h+\|Q_{h+1}^*(\cdot)\|_{L^\infty})\\
    &\le \Lambda_d(\varepsilon_h+H-h) \\
    &\le \Lambda_d(H+1)
    \end{align*}
    
    where we have used the definition of $\widetilde V_{h+1}$ for the first inequality, the fact that $\beta_++\beta_-=\Lambda_d$ and the definition of $\widehat V_{h+1}$ for the third one, the reverse triangular inequality for the fourth one, the inductive assumption for the fifth one, and for the last one the bound on $\|Q_{h+1}^*(\cdot)\|_{L^\infty}$. Having assumed that the reward at any step is bounded in $[0,1]$, the whole random variable 
    
    $$Y[z_h]:=\widetilde r_{h+1}[z_h] + \widetilde V_{h+1}[z_h]-\E[\widetilde r_{h+1}[z_h]+\widetilde V_{h+1}[z_h] | z_h]$$
    
    is bounded in $[-\sigma, \sigma]$, with $\sigma:=1+\Lambda_d(H+1)$, so it is at most $\sigma-$subgaussian. 
    Before using these two results to apply standard bound for linear regression, we have to introduce a discretization of the state action space $\Zs$. This is going to be carachterized by the following quantities:
    \begin{itemize}
        \item $k$ the number of points in the discretization.
        \item $\varepsilon_k$ the minimum value such that we can make a $\varepsilon_k-$cover of $\Zs$ with $k$ points.
        \item $\Zs_{\varepsilon_k}$ the discretization itself.
    \end{itemize}

    At this point, the setting allows us to apply Proposition \ref{prop:regression}:
    \begin{enumerate}
        \item We take $\Xs$ to be the set $\Xs_\varphi$ defined in Algorithm \ref{alg:DATA2}, while $\Xs'$ the subset of $\Xs$ given by $\Xs':=\{\bm \varphi_N(s,a):\ (s,a)\in \Zs_{\varepsilon_k}\}$. By definition, we have $|\Xs'|=k$.
        \item The mean objective $\E[\widetilde r_{h+1}[z_h] + \widetilde V_{h+1}[z_h] | z_h]$ corresponds to
        $$D_N*\mathcal T_h \widehat Q_{h+1}(z_h)=\bm \varphi_N(z_h)^\top\bm \theta_h,$$
        for some unknown $\bm \theta_h$, as the kernel $D_N$ projects all the functions in the span of $\bm \varphi_N(\cdot)$. Remark that $\widehat Q_{h+1}$ is independent of the data used for regression at step $h$ since it is computed with independent data at step $h+1$ (this is only possible with a generative model), so $\bm\theta_h$ can be treated as a fixed regression target in the backward induction.
        \item The noise $Y[z_h]$ is $\sigma-$subgaussian.
    \end{enumerate}

    From Proposition \ref{prop:regression}, for the estimated $\widehat{\bm \theta}_h$, we have
    \begin{equation}
        \Prob\left (\sup_{\bm x\in \Xs'}|\langle \bm \theta_h, \bm x\rangle - \langle  \widehat {\bm \theta}_h, \bm x\rangle| > \sqrt{\log(2k/\delta)}\sup_{\bm x\in \Xs'}\|\bm x\|_{V_n^{-1}}\sigma\right)\le \delta,\label{eq:yin}
    \end{equation}

    and

    \begin{equation}
        \Prob\left (\sup_{\bm x\in \Xs}|\langle \bm \theta_h, \bm x\rangle - \langle  \widehat {\bm \theta}_h, \bm x\rangle| > \sqrt{n\log(2n/\delta)}\sup_{\bm x\in \Xs}\|\bm x\|_{V_n^{-1}}\sigma\right)\le \delta.\label{eq:yang}
    \end{equation}

    where the matrix $V_n$ corresponds to $V_n=\sum_{z_h\in \mathcal D_h} \bm \varphi_N(z_h) \bm \varphi_N(z_h)^\top.$ 
    
    At this point, thanks to the fact that the dataset $\mathcal D_h$ is an optimal design for $\mathcal X$, we can apply Theorem \ref{thm:KW}, which ensures

    $$\sup_{\bm x\in \Xs'}\|\bm x\|_{V_n^{-1}}\le \sup_{\bm x\in \Xs}\|\bm x\|_{V_n^{-1}}\le \frac{\sqrt{2\widetilde N}}{\sqrt{n_{\text{tot}}}}.$$

    Here, note that $\sqrt{2\widetilde N}$ comes from the optimal design, and the denominator from the fact that for each point $z$ in the support of $\rho$ we are taking $\lceil n_{\text{tot}}\rho(z)\rceil$ samples in Algorithm \ref{alg:DATA2}. 
    Substituting into the previous equations results in 

    \begin{equation}
        \Prob\left (\sup_{\bm x\in \Xs'}|\langle \bm \theta_h, \bm x\rangle - \langle  \widehat {\bm \theta}_h, \bm x\rangle| > \frac{\sqrt{2\widetilde N\log(2k/\delta)}\sigma}{\sqrt{n_{\text{tot}}}}\right)\le \delta,\label{eq:yin2}
    \end{equation}

    and

    \begin{equation}
        \Prob\left (\sup_{\bm x\in \Xs}|\langle \bm \theta_h, \bm x\rangle - \langle  \widehat {\bm \theta}_h, \bm x\rangle| > \frac{\sqrt{2\widetilde Nn\log(2n/\delta)}\sigma}{\sqrt{n_{\text{tot}}}}\right)\le \delta.\label{eq:yang2}
    \end{equation}

    By design of the algorithm, $\widehat Q_h(\cdot,\cdot)= \bm \varphi_N(\cdot,\cdot)^\top \widehat{\bm \theta}_h$ and, by definition, $D_N*\mathcal T_h \widehat Q_{h+1}(\cdot,\cdot)=\bm \varphi_N(\cdot,\cdot)^\top {\bm \theta}_h$. Thus, the same equations write,

    \begin{equation}
        \Prob\left (\sup_{z\in \Zs_{\varepsilon_k}}|\widehat Q_h(z)-D_N*\mathcal T_h \widehat Q_{h+1}(z)| > \frac{\sqrt{2\widetilde N\log(2k/\delta)}\sigma}{\sqrt{n_{\text{tot}}}}\right)\le \delta,\label{eq:yin3}
    \end{equation}

    and

    \begin{equation}
        \Prob\left (\|\widehat Q_h(\cdot)-D_N*\mathcal T_h \widehat Q_{h+1}(\cdot)\|_{L^\infty(\Zs)} > \frac{\sqrt{2\widetilde Nn\log(2n/\delta)}\sigma}{\sqrt{n_{\text{tot}}}}\right)\le \delta.\label{eq:yang3}
    \end{equation}

    \textbf{Part \rom{2}: Prove that $\widehat Q_h$ is Lipschitz continuous with a bounded constant.} 
    The two Equations \eqref{eq:yin3} and \eqref{eq:yang3} have very different roles in what follows. For now, we need only to have a universal upper bound on the infinity norm of the $\widehat Q_h$ function, so we forget about Equation \eqref{eq:yin3} and focus on equation \eqref{eq:yang3}. Indeed, the latter result allows us to prove that, with probability at least $1-\delta$,
    \begin{align}
        \|\widehat Q_h(\cdot)-\widetilde Q_h(\cdot)\|_{L^\infty(\Zs)} &\le \|\widehat Q_h(\cdot)-D_N*\mathcal T_h \widehat Q_{h+1}(\cdot)\|_{L^\infty(\Zs)}+\|\widetilde Q_h(\cdot)-D_N*\mathcal T_h \widehat Q_{h+1}(\cdot)\|_{L^\infty(\Zs)} \nonumber\\
        & = \|\widehat Q_h(\cdot)-D_N*\mathcal T_h \widehat Q_{h+1}(\cdot)\|_{L^\infty(\Zs)} \nonumber\\
        &\ \ \ +\|D_N*\mathcal T_h \widetilde Q_{h+1}(\cdot)-D_N*\mathcal T_h \widehat Q_{h+1}(\cdot)\|_{L^\infty(\Zs)} \nonumber\\
        & \le \|\widehat Q_h(\cdot)-D_N*\mathcal T_h \widehat Q_{h+1}(\cdot)\|_{L^\infty(\Zs)} \nonumber\\
        &\ \ \ +\Lambda_d\|\widetilde Q_{h+1}(\cdot)-\widehat Q_{h+1}(\cdot)\|_{L^\infty(\Zs)} \nonumber\\
        & \le \chi(\widetilde N,n,n_{\text{tot}},\delta)+\Lambda_d\varepsilon \nonumber\\
        & \le \chi(\widetilde N,n,n_{\text{tot}},\delta)+\Lambda_d, \label{eq:boundqqhat}
    \end{align}
    where the second passage follows by definition of $\widetilde Q_{h}$, the third from the bound on the $L^1$ norm of $D_N$ (Equation~\ref{eq:boundD}) and the last from Equation \eqref{eq:yang3}, having defined $\chi(\widetilde N,n,n_{\text{tot}},\delta):=\frac{\sqrt{2\widetilde Nn\log(2n/\delta)}\sigma}{\sqrt{n_{\text{tot}}}}$ and by induction. 
    Note that the previous bound is \textit{not} tight. We will need a much better result to prove the theorem; furthermore, we have not used Equation~\eqref{eq:yin3} yet. Still, we can use the previous result to prove the smoothness of $\widehat Q_h$. First, notice that 

    \begin{align*} Vol(\Zs)\left\|\widehat Q_h(\cdot,\cdot)\right\|_{L^\infty}^2 &\ge \left\|\widehat Q_h(\cdot,\cdot)\right\|_{L^2}^2\\
    &\overset{(Par)}{=} \left\|\sum_{i=1}^{\widetilde N}\bm [\bm {\widehat \theta_h}]_i[\bm \varphi_N]_i(\cdot,\cdot)\right\|_{L^2}^2\\
        &= \sum_{i=1}^{\widetilde N}\bm [\bm {\widehat \theta_h}]_i^2\|[\bm \varphi_N]_i(\cdot,\cdot)\|_{L^2}^2\\
        &= \sum_{i=1}^{\widetilde N}\bm [\bm {\widehat \theta_h}]_i^2\\
        &=\|\bm {\widehat \theta_h}\|_2^2
    \end{align*}
    where the second passage comes from Parseval's theorem, exploiting the fact that the $\widetilde N$ components of $\bm \varphi_{N}$ are all orthogonal in $L^2(\Ss\times \As)$ and normalized to $1$. The final inequality is just a standard $L^2-L^\infty$ norm inequality. In the end this proves $\|\bm {\widehat \theta_h}\|_2\le \sqrt{Vol(\Zs)}\left\|\widehat Q_h(\cdot,\cdot)\right\|_{L^\infty}.$ This result, combined with Equation \eqref{eq:boundqqhat}, brings to
    \begin{align*}
        \|\bm {\widehat \theta_h}\|_2^2 &\le\left\|\widehat Q_h(\cdot,\cdot)\right\|_{L^2}^2\le Vol(\Zs)\left\|\widehat Q_h(\cdot,\cdot)\right\|_{L^\infty}^2,
    \end{align*}
    
    as well as the fact that $\widetilde Q_h$ is bounded by $\frac{(\Lambda_dC_{\mathcal T})^{H-h+1}-\Lambda_dC_{\mathcal T}}{\Lambda_dC_{\mathcal T}-1}$ (theorem \ref{thm:approxQ}), entails that
    \begin{align*}
        \|\bm {\widehat \theta_h}\|_2^2 &\le  Vol(\Zs)\|\widehat Q_h(\cdot,\cdot)\|_{L^\infty}^2\\
        &\le Vol(\Zs)\left(\|\widehat Q_h(\cdot)-\widetilde Q_h(\cdot)\|_{L^\infty}+\|\widetilde Q_h(\cdot,\cdot)\|_{L^\infty}\right)^2\\
        & \le Vol(\Zs)\left(\chi(\widetilde N,n,n_{\text{tot}},\delta)+\Lambda_d+\frac{(\Lambda_dC_{\mathcal T})^{H-h+1}-\Lambda_dC_{\mathcal T}}{\Lambda_dC_{\mathcal T}-1}\right)^2=:\chi_2(\widetilde N,n,n_{\text{tot}},\delta,h).
    \end{align*}

    This equation is fundamental, as it allows us to bound the $\mathcal C^{\nu,1}$-norm of $\widehat Q_h(\cdot,\cdot)$ everywhere. Indeed,

    \begin{align}
        \|\widehat Q_h(\cdot,\cdot)\|_{\mathcal C^{\nu,1}} &= \max_{| \alpha|\le \nu+1}\|D^{\alpha} \widehat Q_h(\cdot,\cdot)\|_{L_\infty}\nonumber\\
        &= \max_{| \alpha|\le \nu+1}\|D^{\alpha} \bm \varphi_N(\cdot,\cdot)^\top \bm{\widehat \theta}_h\|_{L_\infty}\nonumber\\
        &\overset{CS}{\le} \max_{| \alpha|\le \nu+1}\|\|D^{\alpha} \bm \varphi_N(\cdot,\cdot) \|_2\|\bm{\widehat \theta}_h\|_2\|_{L_\infty}\nonumber\\
        &\le \chi_2(\widetilde N,n,n_{\text{tot}},\delta,h)\max_{| \alpha|\le \nu+1}\|\|D^{\alpha} \bm \varphi_N(\cdot,\cdot)^\top \|_2\|_{L_\infty}\nonumber\\
        & \le \chi_2(\widetilde N,n,n_{\text{tot}},\delta,h)\sqrt{\widetilde N}\max_{| \alpha|\le \nu+1}\|\|D^{\alpha} \bm \varphi_N(\cdot,\cdot) \|_\infty\|_{L_\infty}\nonumber\\
        &\le N^{\nu+1}\sqrt{\widetilde N}\chi_2(\widetilde N,n,n_{\text{tot}},\delta,h)\label{eq:qbound}.
    \end{align}
    Where the third passage is by the Cauchy-Schwartz inequality, the fourth one from the bound on $\|\bm{\widehat \theta}_h\|_2$ and the last one from the fact that the derivatives of order $|\bm \alpha|\le \nu+1$ of the Fourier features of degree $N$ do not exceed $N^{\nu+1}$ (by just computing the derivatives, $\sin(Nx)\to N\cos(Nx)\to -N^{2}\sin(Nx),\dots$).

    \textbf{Part \rom{3}: Extend equation \eqref{eq:yang3} to all $\Zs$ trough the Lipschitzness of $\widehat Q_h$.}

    At this point, we have to fix the previously mentioned cover $\Zs_{\varepsilon_k}$ of $\Zs$. It is well-known \citep{wu2017lecture} that to construct an $\varepsilon_k$-cover of $\Zs=[-1,1]^d$ we need
    $k=3^d \varepsilon_k^{-d}$ points. In our case, we are interested in a cover that contains a near-maximum point for the function $\widehat Q_h(\cdot)-D_N*\mathcal T_h \widehat Q_{h+1}(\cdot)$, therefore the value of $\varepsilon$ must be determined based on the Lipschitz constant of the latter function. We have proved in Equation \ref{eq:qbound} that $\|\widehat Q_h(\cdot,\cdot)\|_{\mathcal C^{\nu,1}} \le N^{\nu+1}\sqrt{\widetilde N}\chi_2(\widetilde N,n,n_{\text{tot}},\delta,h)$, so

    \begin{align*}
        \|\widehat Q_h(\cdot)-D_N*\mathcal T_h \widehat Q_{h+1}(\cdot)\|_{\mathcal C^{\nu,1}} &\le \|\widehat Q_h(\cdot)\|_{\mathcal C^{\nu,1}}+\|D_N*\mathcal T_h \widehat Q_{h+1}(\cdot)\|_{\mathcal C^{\nu,1}}\\
        &\le \|\widehat Q_h(\cdot)\|_{\mathcal C^{\nu,1}}+\Lambda_d\|\mathcal T_h \widehat Q_{h+1}(\cdot)\|_{\mathcal C^{\nu,1}}\\
        &\le \|\widehat Q_h(\cdot)\|_{\mathcal C^{\nu,1}}+\Lambda_dC_{\mathcal T}(\|\widehat Q_{h+1}(\cdot)\|_{\mathcal C^{\nu,1}}+1)\\
        &\le (1+\Lambda_dC_{\mathcal T})\sqrt{\widetilde N}\chi_2(\widetilde N,n,n_{\text{tot}},\delta,h)+\Lambda_dC_{\mathcal T}\\
        & \le (1+2\Lambda_dC_{\mathcal T})\sqrt{\widetilde N}\chi_2(\widetilde N,n,n_{\text{tot}},\delta,h).
    \end{align*}

    where the second passage follows by Equation \eqref{eq:Dprop} and the third by definition of $C_{\mathcal T}$, 
    while the last one follows from Equation \eqref{eq:qbound}. As the previous display bounds, by definition, the Lipschitz constant of the same function, we can just take
    $$L_{\text{Bell}}:= (1+2\Lambda_dC_{\mathcal T})N^{\nu+1}\sqrt{\widetilde N}\chi_2(\widetilde N,n,n_{\text{tot}},\delta,h),$$
    and say that the Bellmann error $\widehat Q_h(\cdot)-D_N*\mathcal T_h \widehat Q_{h+1}(\cdot)$ is Lipschitz continuous with this constant.
    Therefore, if we want to bound its variation with $\varepsilon$, we need a $\frac{\varepsilon}{L_{\text{Bell}}}$-cover, which requires

    $$k=\left (\frac{3L_{\text{Bell}}}{\varepsilon}\right)^d$$

    points. In particular, we want to fix

    $$\varepsilon = \frac{\sqrt{2\widetilde N\log(2k/\delta)}\sigma}{\sqrt{n_{\text{tot}}}},$$

    so the number of points required is

    $$k=\left (\frac{3L_{\text{Bell}}\sqrt{n_{\text{tot}}}}{\sqrt{2\widetilde N\log(2k/\delta)}\sigma}\right)^d.$$

    With this choice of $k$, we have, with probability $1-2H\delta$ (the coefficient $2H$ is from a union bound, since both Equation \eqref{eq:yin3} and \eqref{eq:yang3} have to hold for every $h$),

    \begin{align}
        \|\widehat Q_h(\cdot)-D_N*\mathcal T_h \widehat Q_{h+1}(\cdot)\|_{L^\infty(\Zs)}&\le \sup_{z\in \Zs_{\varepsilon_k}}|\widehat Q_h(z)-D_N*\mathcal T_h \widehat Q_{h+1}(z)| + \varepsilon \nonumber\\
        & \overset{\eqref{eq:yin3}}{\le} \frac{\sqrt{2\widetilde N\log(2k/\delta)}\sigma}{\sqrt{n_{\text{tot}}}} + \varepsilon \nonumber\\
        & = \label{eq:fullbell}\frac{\sqrt{8\widetilde N\log(2k/\delta)}\sigma}{\sqrt{n_{\text{tot}}}},
    \end{align}

    where the first passage comes from the definition of the cover, the second from Equation \eqref{eq:yin3}, the last one by definition of $\varepsilon$. We have been finally able to bound this pseudo-Bellmann error for all $z\in \Zs$. Even if $k$, the number of points in the cover, depends on $\widetilde N, n,n_{\text{tot}}$, as its dependence takes the form of $\log(k)$, its presence does not affect the sample complexity significantly.

    \textbf{Part \rom{4}: From the Bellmann error to the Q-function error
    }

    At this point, it is not difficult to bound by induction the difference between the estimated Q-function $\widehat Q_h$ and the target $\widetilde Q_h$. Indeed,

    \begin{align}
        \|\widetilde Q_{h}(\cdot)-\widehat Q_h(\cdot)\|_{L^\infty} 
        & \le \|\widetilde Q_{h}(\cdot)-D_N*\mathcal T_h \widehat Q_{h+1}(\cdot)\|_{L^\infty}+\|\widehat Q_h(\cdot)-D_N*\mathcal T_h \widehat Q_{h+1}(\cdot)\|_{L^\infty}\nonumber\\
        & \overset{\eqref{eq:fullbell}}{\le} \|D_N*\mathcal T_h\widetilde Q_{h+1}(\cdot)-D_N*\mathcal T_h \widehat Q_{h+1}(\cdot)\|_{L^\infty}+\frac{\sqrt{8\widetilde N\log(2k/\delta)}\sigma}{\sqrt{n_{\text{tot}}}} \nonumber\\
        & \le \Lambda_d\|\mathcal T_h\widetilde Q_{h+1}(\cdot)-\mathcal T_h \widehat Q_{h+1}(\cdot)\|_{L^\infty}+\frac{\sqrt{8\widetilde N\log(2k/\delta)}\sigma}{\sqrt{n_{\text{tot}}}}\nonumber\\
        & \le \Lambda_d\|\widetilde Q_{h+1}(\cdot)-\widehat Q_{h+1}(\cdot)\|_{L^\infty}+\frac{\sqrt{8\widetilde N\log(2k/\delta)}\sigma}{\sqrt{n_{\text{tot}}}}\nonumber\\
        & \le \Lambda_d\varepsilon_{h+1}+\frac{\sqrt{8\widetilde N\log(2k/\delta)}\sigma}{\sqrt{n_{\text{tot}}}}
    \end{align}

    the second inequality being valid by Equation \eqref{eq:fullbell}, the third one by the usual bound on the $L^1$ norm of $D_N$, the fourth one by the non-expansivity of the Bellmann operator and the last one by the inductive hypothesis.
    We have just proved that, with probability at least $1-2\delta$, the estimation error 
    is bounded as

    \begin{equation}       
    \|\widetilde Q_{h}(\cdot)-\widehat Q_h(\cdot)\|_{L^\infty}\le \Lambda_d\varepsilon_{h+1}+\frac{\sqrt{8\widetilde N\log(2k/\delta)}\sigma}{\sqrt{n_{\text{tot}}}}.\label{eq:qdifffinal}
    \end{equation}

    This term can be naturally coupled with the approximation term between $\widetilde Q_h$ and $Q^*$, which is bounded, by theorem \ref{thm:approxQ} with

    $$\|\widetilde Q_h- Q_h^*\|_{L^\infty}\le \Psi_h N^{-\nu-1}.$$

    At this point, we set $N$ and $n_{\text{tot}}$ in this way (note that $k$ contains terms that are polynomial in $N,n_{\text{tot}}$ but, appearing only in the logarithm, this has a negligible effect):

    $$N=\left \lceil \left ( \frac{\Psi_h}{\Lambda_d\varepsilon_{h+1}} \right )^{\frac{1}{\nu+1}}\right \rceil\qquad n_{\text{tot}}=\left \lceil  \left (\frac{\sigma \sqrt{8\widetilde N\log(2k/\delta)}}{\Lambda_d\varepsilon_{h+1}}\right )^2\right\rceil,$$

    so that $\|\widetilde Q_h- \widehat Q_h\|_{L^\infty}\le 2\Lambda_d\varepsilon_{h+1}\le \varepsilon_{h}$, and by triangular inequality $$\|\widehat Q_h- Q_h^*\|_{L^\infty}\le \Lambda_d\varepsilon_{h+1} + \|\widetilde Q_h- \widehat Q_h\|_{L^\infty} \le 3\Lambda_d\varepsilon_{h+1} = \varepsilon_{h},$$

    as needed to complete the inductive step. This has been done with a number of samples that is at most (proposition \ref{prop:totalsamp})

    $$2(n_{\text{tot}}+4{\widetilde N}\log\log({\widetilde N}))=\bigot(n_{\text{tot}})=\bigot \left (\sigma^2\Psi_h^{\frac{d}{2(\nu+1)}}\varepsilon_h^{-\frac{d+2(\nu+1)}{\nu+1}}\right ),$$

    which, in terms of the original $\varepsilon$, corresponds to

    $$\bigot(n_{\text{tot}})=\bigot \left (\sigma^2\Psi_h^{\frac{d}{2(\nu+1)}}(3\Lambda_d)^{h\frac{d+2(\nu+1)}{\nu+1}}\varepsilon^{-\frac{d+2(\nu+1)}{\nu+1}}\right ).$$

    Note that we have chosen $N$ depending on $h$, which should not in the definition of our algorithm. Still, for the algorithm we can just take the highest value of $N$ among every step $h$, and the corresponding $n_{\text{tot}}$. As the dependence on $\varepsilon$ is always of $\varepsilon^{-\frac{d+2(\nu+1)}{\nu+1}}$, this does not compromise the result.    
    
\end{proof}

Now, we can finally proceed to the proof of the main result.

\mainthm*
\begin{proof}
    We use theorem \ref{thm:inductor} with $\varepsilon\gets \varepsilon/2$, which ensures that under the same probability we are able to learn a function $\widehat Q_{h}$ such that $\| Q_{h}^*-\widehat Q_{h}\|_{L^\infty}<(1/2)\varepsilon/(3\Lambda_d)^h$.
    Since the policy outputted by algorithm \ref{alg:LSVI} is just greedy with respect to $\widehat Q$, we have

    \begin{align*}
        V_h^*(s) - V_h^{\widehat \pi}(s)&= Q_h^*(s,\pi^*(s)) - Q_{h}^{\widehat \pi}(s,\widehat \pi(s)) \\
        & = \underbrace{Q_h^*(s,\pi^*(s)) - Q_h^*(s,\widehat \pi(s))}_{I} + \underbrace{Q_h^*(s,\widehat \pi(s)) - Q_{h}^{\widehat \pi}(s,\widehat \pi(s))}_{II}.
    \end{align*}

    The first part can be easily bounded in this way:
    \begin{align*}
        I&\le \|Q_h^*(\cdot,\cdot) - \widehat Q_{h}(\cdot,\cdot)\|_{L^\infty} + \widehat Q_{h}(s,\pi^*(s)) - \widehat Q_{h}(s,\widehat \pi(s))\\
        &\le \|Q_h(\cdot,\cdot) - \widehat Q_{h}(\cdot,\cdot)\|_{L^\infty},
    \end{align*}
    as the second term is negative by definition of $\widehat \pi$.
    Meanwhile, the second one is bounded by induction:
    \begin{align*}
        II&\le \E_{s'\sim p_h(\cdot|s,\widehat \pi(s))}[V_{h+1}^*(s') - V_{h+1}^{\widehat \pi}(s')]\\
        &\le \|V_{h+1}^*(\cdot) - V_{h+1}^{\widehat \pi}(\cdot)\|_{L^\infty}.
    \end{align*}

    In this way, we have proved that $\|V_h^*(\cdot) - V_h^{\widehat \pi}(\cdot)\|_{L^\infty}\le \|Q_h(\cdot,\cdot) - \widehat Q_{h}^*(\cdot,\cdot)\|_{L^\infty}+\|V_{h+1}^*(\cdot) - V_{h+1}^{\widehat \pi}(\cdot)\|_{L^\infty}$. This implies that

    \begin{equation}
        \|V_h^*(\cdot) - V_h^{\widehat \pi}(\cdot)\|_{L^\infty}\le \sum_{h'=h}^H\|Q_{h'}^*(\cdot,\cdot) - \widehat Q_{h'}(\cdot,\cdot)\|_{L^\infty}\le \frac{1}{2}\sum_{h'=h}^H\varepsilon/(3\Lambda_d)^h.
        \label{eq:standard}
    \end{equation}

    In particular, for $h=1$ we get $\|V_1^*(\cdot) - V_1^{\widehat \pi}(\cdot)\|_{L^\infty}\le \frac{1}{2}\frac{\varepsilon}{1-3\Lambda_d}\le \varepsilon.$ This corresponds to the definition of $\varepsilon-$optimal policy.
    
\end{proof}

\section{Proofs from section \ref{sec:tight}}\label{app:tight}

In this section we prove that every sample complexity bound for algorithms in the Lipschitz MDP setting must grow exponentially with the time horizon $H$. Lipschitz MDPs assume that the transition function of the model, as well as the reward function, are Lipschitz continuous with constants $L_p$ and $L_r$, respectively. Mathematically, the condition on the reward function writes as, for every $ h\in [H],\ s,s'\in \Ss,\ a,a'\in \As$:
$$|R_h(s,a)-R_h(s',a')|\le L_r(\|s-s'\|_2+\|a-a'\|_2).$$
While for the transition function, which maps a state-action pair into a probability distribution, we bound its difference in the Wasserstein metric $\wass(\cdot,\cdot)$~\citep{rachelson2010locality}:
$$\wass(p_h(\cdot|s,a),p_h(\cdot|s',a'))\le L_p(\|s-s'\|_2+\|a-a'\|_2),$$

The Wasserstein metric is a notion of distance for probability measures on metric spaces. For two measures $\mu,\zeta$ on a metric space $\Omega$, it is defined as $\wass(\mu,\zeta):=\sup_{f\in \mathcal C^{0,1}: \|f\|_{\mathcal C^{0,1}}=1} \int_\Omega f(\omega) d(\mu-\zeta)(\omega)$.

The proof strategy is the following: we start from an instance of a Lipschitz bandit problem with a Lipschitz constant that is exponential in $H$, and show that this can be reduced to a standard Lipschtz MDP (where all Lipschitz constants are independent on $H$). Since it has been shown that the sample complexity in a Lipschitz bandit problem is proportional to the Lipschitz constant, this shows that the regret of the Lipschitz MDP is also exponential in $H$. 

\begin{restatable}{thm}{Hexp}
    The minimum number of samples necessary to learn an $\varepsilon-$optimal policy, with probability at least $1-\delta$, in a $(L_p,L_r)-$Lipschitz MDP satisfies
    $n=\Omega(L_p^{d(H-2)}\varepsilon^{-d-2})$.
\end{restatable}

\begin{proof}
Let $f:[-1,1]^d\to [-1,1]$ be an $2L_p^{H-2}-$Lipschitz function and $\eta$ a noise bounded in $[-1,1]$.

Define $\widetilde f := (L_p^{-H+2}/2)f, \widetilde \eta := (L_p^{-H+2}/2)\eta$, so that $\widetilde f:[-1,1]^d\to [-L_p^{-H+2}/2,L_p^{-H+2}/2]$ is a $1/2-$Lipschitz function and $\widetilde \eta$ a noise bounded in $[-L_p^{-H+2}/2,L_p^{-H+2}/2]$. Define the following MDP:
\begin{itemize}
    \item The state and action space coincide: $\Ss = \As = [-1,1]^{d/2}$. In this way, $\Ss\times \As = [-1,1]^d$
    \item The starting state is $[0,\dots 0]$ almost surely.

    \item The transition function is defined in the following way:
    \begin{itemize}
        \item For $h=1$, $p_1(s'|s,a)=\delta(s'=a)$, so that the first action becomes the second state.
        \item For $h=2$, $p_2(s'|s,a)=\delta(s'^{(1)}=\widetilde f(s,a)+\widetilde \eta)\prod_{i=2}^{d/2}\delta_0({s'}^{(i)})$, meaning that the next state has the first coordinate equal to $\widetilde f(s,a)$ plus the noise $\eta$, and all the other ones set to zero. Note that this is coherent with the definition of $f$, which goes $[-1,1]^{d/2}=\Ss \times\As \to \R$. 
        \item For $h=3,\dots H$ we have $p_h(s'|s,a)=\delta(s'=L_ps)$, so that the next state is the previous one times a constant (note that, by the bounds on $\widetilde f,\widetilde \eta$, the state never hits the boundary).
    \end{itemize}
    \item The reward function $r_h$ is zero for the first $H-1$ time steps, and $r_H(s,a)=s^{(1)}$, the first component of the state.
\end{itemize}
By definition, it is easy to check that the MDP is Lipschtz with $L_p=L_p$ and $L_r=1$.

In this very peculiar MDP, where only the first two actions $a_1$ and $a_2$ matter, the return can be expresses as a function of them. Precisely, since the reward is only given at the last time step, we have

$$\text{Return}(a_1,a_2)=L_p^{H-2}(\widetilde f(a_1,a_2)+\widetilde \eta)=\frac{1}{2}(f(a_1,a_2)+\eta).$$

In this way, we have shown that the return for this Lipschitz MDP corresponds exactly to the feedback in the $d$-dimensional Lipschitz bandit problem with reward function $ f/2$ (which is $L_p^{H-2}$-LC) and noise $\eta/2$. 

This shows that any Lipschitz bandit problem with Lipschtz constant $L_P^{H-2}$ can be reduced to to a Lipschtz MDP with constants bounded independently of $H$. Therefore, the sample complexity on the latter problem is at most as high as the one of the former one. As the sample complexity of the latter problem is well-known to be of order $\Omega(L^d\varepsilon^{-2-d})=\Omega(L_p^{(H-2)d}\varepsilon^{-2-d})$, the proof is complete.
\end{proof}

\section{Auxiliary Results}

In the main proof, we need this bound for regression, which is a variant of the classical results for ordinary least squares with fixed design.

\begin{prop}\label{prop:regression}
    Assume to have a dataset of the form $\{\bm x_t,y_t\}_{t=1}^n$, where $\bm x_t\in \Xs$ and $y_t\in \R$, with $y_t=\bm x_t^\top \bm \theta + \xi_t$, for some unknown $\bm \theta \in \R^{\widetilde N}$, and a random noise $\{\xi_t\}_{t=1}^n$ with independent $\sigma-$subgaussian components and independent from $\{\bm x_t\}_{t=1}^n$. Suppose that $\bm x_t$ for $t=1,\dots n$ span the whole space $\R^{\widetilde N}$ and define the LS estimator 
    $$\widehat {\bm \theta}:=\left (\sum_{t=1}^n \bm x_t \bm x_t^\top \right )^{-1} \sum_{t=1}^n \bm x_t y_t.$$
    Then, for every $\delta>0$, fixing an arbitrary subset $\Xs'\subset \Xs$ with $|\Xs|=k$ we have
    \begin{itemize}
        \item $\Prob\left (\sup_{\bm x\in \Xs'}|\langle \bm \theta, \bm x\rangle - \langle  \widehat {\bm \theta}, \bm x\rangle| > \sqrt{\log(2k/\delta)}\sup_{\bm x\in \Xs'}\|\bm x\|_{V_n^{-1}}\sigma\right)\le \delta.$
        \item $\Prob(\sup_{\bm x\in \Xs} |\langle \bm \theta, \bm x\rangle - \langle  \widehat {\bm \theta}, \bm x\rangle| > \sqrt{n\log(2n/\delta)}\sup_{x\in \Xs}\|\bm x\|_{V_n^{-1}}\sigma)\le \delta.$
    \end{itemize}
\end{prop}
\begin{proof}
    By simplicity, let us call
    $$V_n=\sum_{t=1}^n \bm x_t \bm x_t^\top.$$
    Let $\bm x\in \Xs$. We have
    \begin{align*}
        \langle \bm \theta, \bm x\rangle - \langle  \widehat {\bm \theta}, \bm x\rangle & = \langle \bm \theta-\widehat {\bm \theta}, \bm x\rangle\\
        & = \left \langle V_n^{-1}V_n \bm \theta-V_n^{-1}\sum_{t=1}^n \bm x_t y_t, \bm x\right\rangle\\
        & = \left \langle V_n^{-1}\sum_{t=1}^n \bm x_t \bm x_t^\top\bm \theta-V_n^{-1}\sum_{t=1}^n \bm x_t(\bm x_t^\top\bm \theta+\xi_t), \bm x\right\rangle\\
        & = \left \langle -V_n^{-1}\sum_{t=1}^n \bm x_t\xi_t, \bm x\right\rangle\\
        & = -\sum_{t=1}^n\left \langle V_n^{-1}\bm x_t, \bm x\right\rangle \xi_t.
    \end{align*}
    Being $\{\xi_t\}_{t}$ a sequence of $\sigma-$subgaussian independent random variables, the sum is also subgaussian conditionally on the data points with a constant given by
    $$\sigma_{n}(\bm x)=\sigma\sqrt{\sum_{t=1}^n\left \langle V_n^{-1}\bm x_t, \bm x\right\rangle^2}.$$
    We can bound the last quantity in the following way
    \begin{align*}
        \frac{\sigma_{n}(\bm x)^2}{\sigma^2}&=\sum_{t=1}^n\left \langle V_n^{-1}\bm x_t, \bm x\right\rangle^2\\
        &=\Tr \left (\sum_{t=1}^n\left \langle V_n^{-1}\bm x_t, \bm x\right\rangle^2\right)\\
        & =\sum_{t=1}^n \Tr \left (\left\langle V_n^{-1}\bm x_t, \bm x\right\rangle^2\right)\\
        & =\sum_{t=1}^n \Tr \left (\bm x^\top V_n^{-1}\bm x_t\bm x_t^\top V_n^{-1}\bm x\right)\\
        & = \Tr \left (\bm x^\top V_n^{-1}\sum_{t=1}^n\bm x_t\bm x_t^\top V_n^{-1}\bm x\right)\\
        & = \Tr \left (\bm x^\top V_n^{-1}\bm x\right)=\|\bm x\|_{V_n^{-1}}^2.
    \end{align*}
    This entails that $\sigma_{n}(\bm x)=\|\bm x\|_{V_n^{-1}}\sigma$. Now, fixing any $\delta>0$ and applying Hoeffding's inequality, we have

    $$\Prob(\langle \bm \theta, \bm x\rangle - \langle  \widehat {\bm \theta}, \bm x\rangle > \sqrt{\log(1/\delta)}\|\bm x\|_{V_n^{-1}}\sigma)\le \delta.$$   

    At this point, by a union bound over $\Xs'$, we have the first part of the statement:

    $$\Prob(\sup_{\bm x\in \Xs'} \langle \bm \theta, \bm x\rangle - \langle  \widehat {\bm \theta}, \bm x\rangle > \sqrt{\log(k/\delta)}\sup_{\bm x\in \Xs'}\|\bm x\|_{V_n^{-1}}\sigma)\le \delta.$$ 

    For the second statement, we have to proceed in a different way. Indeed, by the same argument as before, we have

    \begin{align*}
        \sup_{\bm x\in \Xs} \langle \bm \theta, \bm x\rangle - \langle  \widehat {\bm \theta}, \bm x\rangle &\le -\sum_{t=1}^n\left \langle V_n^{-1}\bm x_t, \bm x\right\rangle \xi_t\\
        & \overset{CS}\le \sqrt{\sum_{t=1}^n\left \langle V_n^{-1}\bm x_t, \bm x\right\rangle^2}\sqrt{\sum_{t=1}^n \xi_t^2}\\
        & \le \sup_{\bm x\in \Xs}\|\bm x\|_{V_n^{-1}}\sqrt{\sum_{t=1}^n \xi_t^2}.
    \end{align*}

    Here, the second inequality is due to the Cauchy-Schwartz inequality, and the third one to the same inequality that we have used to prove the first statement. At this point note that

    $$\sqrt{\sum_{t=1}^n \xi_t^2}\le \sqrt n \max_{t}|\xi_t|,$$

    which, being the maximum of a sequence of $\sigma-$subgaussian random variables, is bounded by $\sqrt n\sqrt{\log(2n/\delta)}\sigma$ with probability $1-\delta$. This entails that

    $$\Prob(\sup_{\bm x\in \Xs} |\langle \bm \theta, \bm x\rangle - \langle  \widehat {\bm \theta}, \bm x\rangle| > \sqrt{n\log(2n/\delta)}\sup_{x\in \Xs}\|\bm x\|_{V_n^{-1}}\sigma)\le \delta,$$ 

    completing the second part of the proof.
    
\end{proof}

\section{Computational complexity}\label{app:compu}

The computational complexity of our algorithm can be divided into the one for the preliminary operations and the one for learning from the samples (after line 10). Algorithmically, line \ref{algline:x} is performed by fixing $\varepsilon'$ and constructing an $\varepsilon'-$grid of the state-action space, which amounts to $k=(2/\varepsilon')^d$ points. As the matrix norm is continuous, this introduces an error that can be made arbitrarily small with $\varepsilon'$. Once this is done, the optimal design (line \ref{algline:qo}) can be computed in just $k\widetilde N^2$ steps (see \cite{lattimore2020learning} just before theorem 4.4). The learning part requires to solve a linear regression problem, which is well known to take $n\widetilde N^2+\widetilde N^3$ steps, where $n$ is the number of samples. Therefore, the total computational complexity is roughly of order $k\widetilde N^2+n\widetilde N^2+\widetilde N^3$. Even if $k$ is exponential in $d$, note that this term is dominated by the number of samples $n$, which is exponential in $dH$. Note that $n$ (i.e., the sample complexity) must be exponential in $dH$  \textit{for any algorithm}, as the lower bound (theorem \ref{thm:vecio}) shows. Therefore, our computational complexity grows as $n\widetilde N^2$. In contrast, note that other algorithms for other very general RL settings often require access to an optimization oracle \citet{jin2020provably,du2021bilinear}. Even in the much more restricive case of low-rank MDPs, the state of the art was computationally efficient \cite{uehara2021representation} until very recently, when an efficient counterpart  \cite{mhammedi2024efficient} was proved to work.

\end{document}